\documentclass[12pt,letterpaper]{article}
\usepackage{amssymb,amsmath,amsthm,enumerate,nicefrac}
\usepackage{graphicx, color}
    \usepackage[driverfallback=hypertex,pagebackref=true,colorlinks]{hyperref}
    \hypersetup{linkcolor=[rgb]{.7,0,0}}
    \hypersetup{citecolor=[rgb]{0,.7,0}}
    \hypersetup{urlcolor=[rgb]{.7,0,.7}}
\usepackage{geometry}
\usepackage{epstopdf}
\AtBeginDocument{%
  \addtolength\abovedisplayskip{-0.25\baselineskip}%
  \addtolength\belowdisplayskip{-0.15\baselineskip}%
}

\usepackage{float}
\floatstyle{ruled}
\newfloat{algorithm}{tbp}{loa}
\providecommand{\algorithmname}{Algorithm}
\floatname{algorithm}{\protect\algorithmname}

\usepackage[titletoc,title]{appendix}
\usepackage{url}

\theoremstyle{plain}
\newtheorem{theorem}{Theorem}
\newtheorem{lemma}{Lemma}[section]
\newtheorem{claim}[lemma]{Claim}

\newtheorem{definition}[lemma]{Definition}
\newtheorem{corollary}[theorem]{Corollary}

\def\XX{ X }
\def\AA{ A }

\def\N{{\mathbb {N}}}
\def\Reals{{\mathbb {R}}}

\newcommand{\eps}{\epsilon}
\newcommand{\Ex}{\mathop{\bf E\/}}
\def\poly{{\rm {poly}}}

\def\U{{\mathcal {U}}}

\def\T{{\mathcal {T}}}

\def\P{{\mathbb {P}}}
\newcommand{\norm}[1]{\left\lVert#1\right\rVert}
\newcommand{\inner}[1]{\langle#1\rangle}

\def\Bad{\mathrm{Bad}}
\def\Sig{\mathrm{Sig}}

\usepackage{color}

\newcommand{\kk}{k}
\newcommand{\rr}{r}

\newcommand{\F}{\mathbb{F}}
\newcommand{\bias}{\mathrm{bias}}

\title{Extractor-Based Time-Space Lower Bounds for Learning}
\author{Sumegha Garg\thanks{Department of Computer Science, Princeton University.}
\and
Ran Raz%
\thanks{Department of Computer Science, Princeton University. Research supported by the Simons Collaboration on Algorithms and Geometry and by the National Science Foundation grant No. CCF-1412958.}
\and
Avishay Tal\thanks{Institute for Advanced Study, Princeton, NJ.
    Research supported by the Simons Collaboration on Algorithms and Geometry and by the National Science Foundation grant No. CCF-1412958.}}

\date{}
\begin{document}
\maketitle

\begin{abstract}
A matrix $M: \AA \times \XX \rightarrow \{-1,1\}$ corresponds to the following learning problem:
An unknown element $x \in \XX$ is chosen uniformly at random. A learner tries to learn $x$ from a stream of samples,
$(a_1, b_1), (a_2, b_2) \ldots$, where for every $i$, $a_i \in \AA$ is
chosen uniformly at random and $b_i = M(a_i,x)$.

Assume that $k,\ell, r$ are such that any submatrix of $M$ of at least $2^{-k} \cdot |A|$ rows and at least
$2^{-\ell} \cdot |X|$ columns, has a bias of at most $2^{-r}$.
We show that
any learning algorithm for the learning problem corresponding to $M$ requires either a memory
of size at least
$\Omega\left(k \cdot  \ell \right)$, or at least $2^{\Omega(r)}$ samples.
The result holds even if the learner has an exponentially small success probability (of $2^{-\Omega(r)}$).

In particular, this shows that for a large class of learning problems, any learning algorithm requires either a memory
of size at least
$\Omega\left((\log |X|) \cdot  (\log |A|)\right)$  or an exponential number of samples, achieving a tight
$\Omega\left((\log |X|) \cdot  (\log |A|)\right)$
lower bound on the size of the memory,
rather than a bound of
$\Omega\left(\min\left\{(\log |X|)^2,(\log |A|)^2\right\}\right)$
obtained in previous works~\cite{Raz17,MM2}.

Moreover, our result implies all previous memory-samples lower bounds, as well as a number of new applications.

Our proof builds on~\cite{Raz17} that gave a general  technique for proving memory-samples lower bounds.

%

\end{abstract}

\section{Introduction}

Can one prove unconditional lower bounds on the number of samples needed for learning, under memory constraints?
The study of the resources needed for learning, under
memory constraints
was initiated by Shamir~\cite{Shamir} and by
Steinhardt, Valiant and Wager~\cite{SVW}.
While the main motivation for studying this question comes from learning theory, the problem is also relevant to computational complexity and
cryptography~\cite{Raz16, VV, KRT}.

Steinhardt, Valiant and Wager conjectured that any algorithm for learning parities of size $n$
requires either a memory
of size $\Omega(n^2)$ or an exponential number of samples.
This conjecture was proven in~\cite{Raz16}, showing for the first time a learning problem that is infeasible under super-linear memory constraints.
Building on~\cite{Raz16}, it was proved in~\cite{KRT} that learning parities of sparsity $\ell$ is also infeasible under memory constraints that are super-linear in $n$,
as long as $\ell \geq \omega(\log n / \log \log n)$.
Consequently, learning linear-size
DNF Formulas, linear-size Decision Trees and logarithmic-size
Juntas were all proved to be infeasible under super-linear memory
constraints~\cite{KRT} (by a reduction from learning sparse parities).

Can one prove similar memory-samples lower bounds for other learning problems?

As in~\cite{Raz17}, we represent a learning problem by a matrix.
Let $\XX$, $\AA$ be two finite sets of size larger than 1 (where $\XX$ represents the concept-class that we are trying to learn and $\AA$ represents the set of possible samples).
Let $M: \AA \times \XX \rightarrow \{-1,1\}$ be a matrix.
The matrix $M$ represents the following learning problem:
An unknown element $x \in \XX$ was chosen uniformly at random. A learner tries to learn $x$ from a stream of samples,
$(a_1, b_1), (a_2, b_2) \ldots$, where for every $i$, $a_i \in \AA$ is
chosen uniformly at random and $b_i = M(a_i,x)$.

Let $n = \log|\XX|$ and $n' = \log|\AA|$.

A general technique for proving memory-samples lower bounds was given in~\cite{Raz17}.
The main result of~\cite{Raz17} shows that if the norm of the matrix $M$ is sufficiently small, then any learning algorithm for the corresponding learning problem requires either a memory
of size at least $\Omega\left(\left(\min\{n,n'\}\right)^2\right)$, or an exponential number of samples.
This gives a
general memory-samples lower bound that applies for a large class of learning
problems.

Independently of~\cite{Raz17},
Moshkovitz and Moshkovitz also gave a general technique for proving memory-samples lower bounds~\cite{MM1}.
Their initial result was that if $M$ has a (sufficiently strong) mixing property then any learning algorithm for the corresponding learning problem requires
either a memory of size at least
$1.25 \cdot  \min\{n,n'\}$ or
an exponential number of samples~\cite{MM1}.
In a
recent subsequent work~\cite{MM2}, they
improved their result, and obtained a theorem that is very similar to the one proved in~\cite{Raz17}.
(The result of~\cite{MM2} is stated in terms of a combinatorial mixing property, rather than matrix norm. The two notions are closely related (see in particular Corollary~5.1 and Note~5.1 in~\cite{BL})).

\subsubsection*{Our Results}

The results of~\cite{Raz17} and~\cite{MM2} gave
a lower bound of at most $\Omega\left(\left(\min\{n,n'\}\right)^2\right)$ on the size of the memory,
whereas the best that one could hope for, in the information theoretic setting
(that is, in the setting where the learner's computational power is unbounded), is a lower bound of
$\Omega\left(n \cdot n'\right)$, which may be significantly larger in cases where $n$ is significantly larger than~$n'$, or vice versa.

In this work, we build on~\cite{Raz17} and obtain a general memory-samples lower bound that applies for a large class of learning problems and shows that for every problem in that class, any learning algorithm
requires either a memory of size at least $\Omega\left(n \cdot n'\right)$
or an exponential number of samples.

Our result is stated in terms of the properties of the matrix $M$ as a two-source extractor.
Two-source extractors, first studied by Santha and Vazirani~\cite{SV} and Chor and Goldreich~\cite{CG}, are central objects in the study of randomness and derandomization.
We show that
even a relatively weak two-source extractor implies a relatively strong memory-samples lower bound.
We note that two-source extractors have been extensively studied in numerous of works and
there are known techniques for proving that certain matrices are relatively good two-source extractors.

Our main result can be stated as follows (Corollary~\ref{cor:main1}):
Assume that $k,\ell, r$ are such that any submatrix of $M$ of at least $2^{-k} \cdot |A|$ rows and at least
$2^{-\ell} \cdot |X|$ columns, has a bias of at most $2^{-r}$.
Then,
any learning algorithm for the learning problem corresponding to $M$ requires either a memory
of size at least
$\Omega\left(k \cdot  \ell \right)$, or at least $2^{\Omega(r)}$ samples.
The result holds even if the learner has an exponentially small success probability (of $2^{-\Omega(r)}$).

A more detailed result, in terms of the constants involved, is stated in Theorem~\ref{thm:TM-main1} in terms of the properties of $M$ as an $L_2$-Extractor, a new notion that we define in Definition~\ref{definition:l2-extractor}, and is closely related to the notion of  two-source extractor.
(The two notions are equivalent up to small changes in the parameters.)

All of our results (and all applications) hold even if the learner is only required to {\em weakly learn} $x$, that is to output a hypothesis $h: A \rightarrow \{-1,1\}$ with a non-negligible correlation with the $x$-th column of the matrix $M$.
We prove in Theorem~\ref{thm:TM-main2} that even
if the learner is only required to output a hypothesis
that agrees with the $x$-th column of  $M$ on more than
a $1/2 + 2^{-\Omega(r)}$ fraction of the rows, the success
probability is at most $2^{-\Omega(r)}$.

As in~\cite{Raz16,KRT,Raz17}, we model the learning algorithm by a {\it branching program}.
A branching program is the strongest and most general model to use in this context.
Roughly speaking, the model allows a learner with infinite computational power, and bounds only the memory size of the learner and the number of samples used.


As mentioned above, our result implies all previous memory-samples lower bounds, as well as new applications. In particular:
\begin{enumerate}
\item {\bf Parities:}
A learner tries to learn $x=(x_1,\ldots,x_n) \in \{0,1\}^n$, from random linear equations over $\F_2$.
It was proved in~\cite{Raz16} (and follows also from~\cite{Raz17})
that any  learning
algorithm
requires either a memory
of size $\Omega(n^2)$ or an exponential number of samples.
The same result follows  by Corollary~\ref{cor:main1} and the fact that inner product is a good two-source extractor~\cite{CG}.
\item {\bf Sparse parities:}
A learner tries to learn $x=(x_1,\ldots,x_n) \in \{0,1\}^n$ of sparsity $\ell$, from random linear equations over $\F_2$.
In Section~\ref{sec:sparse-parities}, we reprove the main results of~\cite{KRT}. In particular,
any learning algorithm  requires:
\begin{enumerate}
	\item Assuming $\ell \le n/2$: either a memory
of size $\Omega(n \cdot \ell)$ or $2^{\Omega(\ell)}$ samples.
	\item Assuming $\ell \le n^{0.9}$:  either a memory
of size  $\Omega(n \cdot \ell^{0.99})$ or $\ell^{\Omega(\ell)}$ samples.
\end{enumerate}
\item {\bf Learning from sparse linear equations:}
A learner tries to learn $x=(x_1,\ldots,x_n) \in \{0,1\}^n$, from random sparse linear equations,  of sparsity $\ell$, over $\F_2$.
In Section~\ref{sec:sparse-equations}, we prove that
    any learning algorithm  requires:
\begin{enumerate}
	\item Assuming $\ell \le n/2$: either a memory
of size $\Omega(n \cdot \ell)$ or $2^{\Omega(\ell)}$ samples.
	\item Assuming $\ell \le n^{0.9}$:  either a memory
of size  $\Omega(n \cdot \ell^{0.99})$ or $\ell^{\Omega(\ell)}$ samples.
\end{enumerate}
\item {\bf Learning from low-degree equations:}
A learner tries to learn $x=(x_1,\ldots,x_n) \in \{0,1\}^n$, from random multilinear polynomial  equations  of degree at most $d$, over $\F_2$.
In Section~\ref{sec:low-deg-equations}, we prove that if $d\le 0.99 \cdot n$,
    any learning algorithm
requires  either a memory
of size
$\Omega\left( \binom{n}{\le d} \cdot n/d \right)$
or
$2^{\Omega(n/d)}$
samples.
\item {\bf Low-degree polynomials:}
A learner tries to learn an $n'$-variate multilinear polynomial $p$ of degree at most $d$ over $\F_2$,
from random evaluations of $p$ over $\F_2^{n'}$.
In Section~\ref{sec:low-degree-polynomials}, we prove that if $d\le 0.99 \cdot n'$,
    any learning algorithm
requires  either a memory
of size
$\Omega\left( \binom{n'}{\le d} \cdot n'/d \right)$
or
$2^{\Omega(n'/d)}$
samples.
\item {\bf Error-correcting codes:}
A learner tries to learn a codeword from random coordinates:
Assume that $M: \AA \times \XX \rightarrow \{-1,1\}$ is such that for some $|\XX|^{-1} \le  \epsilon < 1$, any pair of different columns of $M$, agree on at least $\tfrac{1-\epsilon}{2} \cdot |A|$ and at most $\tfrac{1+\epsilon}{2} \cdot |A|$ coordinates.
In Section~\ref{sec:sq}, we prove
that any  learning
algorithm for the learning problem corresponding to $M$
requires either a memory
of size $\Omega\big( (\log|\XX|) \cdot (\log(1/\epsilon))\big)$ or
$\big(\tfrac{1}{\epsilon}\big)^{\Omega(1)}$ samples.
We also point to a relation between our results and statistical-query dimension~\cite{K98,BFJKMR}.
\item {\bf Random matrices:}
Let $\XX, \AA$ be finite sets, such that,
$|\AA| \geq (2\log|\XX|)^{10}$ and
$|\XX| \geq (2\log|\AA|)^{10}$.
Let $M: \AA \times \XX \rightarrow \{-1,1\}$ be a random matrix.
Fix $k = \tfrac{1}{2}\log |\AA|$ and $\ell = \tfrac{1}{2}\log |\XX|$.
With very high probability,
any submatrix of $M$ of at least $2^{-k} \cdot |A|$ rows and at least
$2^{-\ell} \cdot |X|$ columns, has a bias of at most $2^{-\Omega(\min\{k,\ell\})}$.
Thus, by Corollary~\ref{cor:main1},
any learning algorithm for the learning problem corresponding to $M$ requires either a memory
of size
$\Omega\left((\log|\XX|) \cdot  (\log|\AA|) \right)$,
or $\big(\min\{|\XX|,|\AA|\}\big)^{\Omega(1)}$ samples.
\end{enumerate}


We note also that our results about learning from sparse linear equations have applications in bounded-storage cryptography. This is similar to~\cite{Raz16,KRT}, but in a different range of the parameters.
In particular, for every $\omega(\log n) \leq \ell \leq n$,
our results give an encryption scheme that requires a private key of length $n$, and  time complexity of
$O(\ell \log n)$ per encryption/decryption of each bit, using a random access machine. The scheme is
provenly and unconditionally secure as long as the attacker uses
at most $o(n\ell)$ memory bits
and the scheme is used at most $2^{o(\ell)}$ times.

\subsubsection*{Techniques}

Our proof follows the lines of the proof of~\cite{Raz17} and builds on that proof. The proof of~\cite{Raz17} considered the norm of the matrix~$M$, and thus essentially reduced the entire matrix to only one parameter. In our proof, we consider the properties of $M$ as
a two-source extractor,
and hence we have three parameters $(k,\ell,r)$, rather than one. Considering these three parameters, rather than one, enables a more refined analysis, resulting in a stronger lower bound with a slightly simpler proof.

A proof outline is given in Section~\ref{sec:overview}.

\subsubsection*{Motivation and Discussion}

Many previous works studied the resources needed for learning, under certain information, communication or memory constraints (see in particular~\cite{Shamir, SVW,Raz16, VV, KRT, MM1,Raz17,MT,MM2} and the many references given there).
A main message of some  of these works is that for some learning problems, access to a relatively large memory is crucial.
In other words, in some cases, learning is infeasible, due to memory constraints.

From the point of view of human learning, such results
may help to explain the importance of memory in cognitive processes.
From the point of view of machine learning, these results imply that a large class of learning algorithms cannot learn certain concept classes. In particular, this applies to any bounded-memory learning algorithm that considers the samples one by one.
In addition, these works are related to computational complexity and have applications in cryptography.


\subsubsection*{Related Work}

Independently of our work, Beame, Oveis Gharan and Yang also gave a combinatorial property of a matrix $M$, that holds for a large class of matrices and implies that any learning algorithm for the corresponding learning problem requires either a memory
of size
$\Omega\left((\log |X|) \cdot  (\log |A|)\right)$  or an exponential number of samples (when $|A| \leq |X|$)~\cite{BOGY}. Their property
is based on a measure of how matrices
amplify the 2-norms of probability distributions that is more refined than the 2-norms of these
matrices.
Their proof also builds on~\cite{Raz17}.

They also show, as an application, tight time-space lower bounds for learning
low-degree polynomials, as well as other applications.

\section{Preliminaries}

Denote by $\U_X: \XX \rightarrow \Reals^+$ the uniform distribution over $\XX$.
Denote by $\log$ the logarithm to  base $2$.
Denote by $\binom{n}{\le k} = \binom{n}{0} + \binom{n}{1} + \ldots + \binom{n}{k}$.

For a random variable $Z$ and an event $E$,
we denote by $\P_Z$ the distribution of the random variables $Z$, and
we denote by $\P_{Z|E}$ the distribution of the random variable $Z$ conditioned on the event $E$.

\subsubsection*{Viewing a Learning Problem as a Matrix}

Let $\XX$, $\AA$ be two finite sets of size larger than 1.
Let $n = \log_2|\XX|$.


Let $M: \AA \times \XX \rightarrow \{-1,1\}$ be a matrix.
The matrix $M$ corresponds to the following learning problem:
There is an unknown element $x \in \XX$ that was chosen uniformly at random. A learner tries to learn $x$ from samples
$(a, b)$, where $a \in \AA$ is chosen uniformly at random and $b = M(a,x)$.
That is, the learning algorithm is given a stream of samples,
$(a_1, b_1), (a_2, b_2) \ldots$, where each~$a_t$ is uniformly distributed and for every $t$, $b_t = M(a_t,x)$.

\subsubsection*{Norms and Inner Products}
Let $p \geq 1$.
For a function
$f: \XX \rightarrow \Reals$,
denote by $\norm{f}_{p}$ the $\ell_p$ norm of $f$, with respect to the  uniform distribution over $\XX$, that is:
$$\norm{f}_{p} =
\left( \Ex_{x \in_R \XX} \left[ |f(x)|^{p} \right] \right)^{1/p}.$$
%
%

For two functions
$f,g: \XX \rightarrow \Reals$, define their inner product with respect to the uniform distribution over $X$ as
$$\langle f,g \rangle =
\Ex_{x \in_R \XX} [ f(x) \cdot g(x) ].$$

For a matrix $M: \AA \times \XX \to \Reals$ and a row $a \in \AA$, we denote by $M_a: \XX \to \Reals$ the function corresponding to the $a$-th row of $M$. Note that for a function $f: \XX \to \Reals$, we have $\inner{M_a, f} = \frac{(M \cdot f)_a}{|X|}$.

\subsubsection*{$L_2$-Extractors and $L_\infty$-Extractors} 

\begin{definition} {\bf 	$L_2$-Extractor:} \label{definition:l2-extractor}
Let $\XX,\AA$ be two finite sets.
A matrix $M: \AA \times \XX \to \{-1,1\}$ is a $(k,\ell)$-$L_2$-Extractor with error $2^{-r}$, if for every non-negative $f : \XX \to \Reals$ with $\frac{\norm{f}_2}{\norm{f}_1} \le 2^{\ell}$ there are at most $2^{-k} \cdot |A|$ rows $a$ in $A$ with
$$
\frac{|\inner{M_a,f}|}{\norm{f}_1}
\ge 2^{-r}\;.
$$
\end{definition}

Let $\Omega$ be a finite set. We denote a distribution over $\Omega$ as a function $f:\Omega \to \Reals^{+}$ such that $\sum_{x\in \Omega}{f(x)} = 1$.
We say that a distribution $f:\Omega \to \Reals^{+}$ has min-entropy $k$ if for all $x\in \Omega$, we have $f(x) \le 2^{-k}$.

\begin{definition}{\bf $L_\infty-$Extractor:} \label{definition:min-extractor}
Let $\XX,\AA$ be two finite sets.
A matrix $M:\AA\times \XX\rightarrow \{-1,1\}$ is a $\left(k ,\ell \sim r \right)$-$L_\infty$-Extractor
if for every distribution $p_x: \XX \to \Reals^{+}$ with min-entropy at least $(\log(|\XX|)-\ell)$
and every distribution $p_a: \AA \to \Reals^{+}$ with min-entropy at least $(\log(|\AA|)-k)$,
$$\bigg|\sum_{a'\in \AA} \sum_{x' \in \XX} p_a(a') \cdot p_x(x') \cdot M(a',x')\bigg| \le 2^{-r}.$$
\end{definition}

\subsubsection*{Branching Program for a Learning Problem} \label{section:def}

In the following definition, we model the learner for the learning problem that corresponds to the matrix $M$, by a {\em branching program}.

\begin{definition} {\bf Branching Program for a Learning Problem:}
A branching program of length $m$ and width $d$, for learning, is a directed (multi) graph with vertices arranged in $m+1$ layers containing at most $d$ vertices each. In the first layer, that we think of as layer~0, there is only one vertex, called the start vertex.
A vertex of outdegree~0 is called a  leaf.
All vertices in the last layer are leaves
(but there may be additional leaves).
Every non-leaf vertex in the program has $2|\AA|$ outgoing edges, labeled by elements
$(a,b) \in \AA \times \{-1,1\}$, with exactly one edge labeled by each such $(a,b)$, and all these edges going
into vertices in the next layer.
Each leaf $v$ in the program is labeled by an element $\tilde{x}(v) \in \XX$, that
we think of as the output of the program on that leaf.

{\bf Computation-Path:} The samples
$(a_1, b_1), \ldots, (a_m, b_m) \in \AA \times \{-1,1\}$
that are given as input,
define a
computation-path in the branching
program, by starting from the start vertex
and following at
step~$t$ the edge labeled by~$(a_t, b_t)$, until reaching a leaf.
The program outputs the label $\tilde{x}(v)$ of the leaf $v$ reached by the computation-path.

{\bf Success Probability:}
The success probability of the program is the probability that $\tilde{x}=x$,
where $\tilde{x}$ is the element that the program outputs, and the probability is over $x,a_1,\ldots,a_m$ (where $x$ is uniformly distributed over $\XX$ and $a_1,\ldots,a_m$ are uniformly distributed over $\AA$, and for every $t$, $b_t = M(a_t,x)$).

\end{definition}

\section{Overview of the Proof} \label{sec:overview}

The proof follows the lines of the proof of~\cite{Raz17} and builds on that proof.

Assume that $M$ is a $(k,\ell)$-$L_2$-extractor with error $2^{-r'}$,
and let $r = \min\{k, \ell, r'\}$.
Let $B$ be a branching program
for the learning problem that corresponds to the matrix $M$.
Assume for a contradiction that $B$ is
of length $m=2^{\epsilon  r}$ and width $d=2^{\epsilon  k \ell}$,
where $\epsilon$ is a  small constant.

We define the {\it truncated-path}, $\T$, to be the same as the computation-path of $B$, except that it sometimes stops before reaching a leaf.
Roughly speaking, $\T$ stops before reaching a leaf if certain ``bad'' events occur.
Nevertheless, we show that the probability that $\T$ stops before reaching a leaf is negligible, so we can think of $\T$ as almost identical to the computation-path.

For a vertex $v$ of $B$, we denote by $E_v$
the event that $\T$ reaches the vertex $v$.
We denote  by $\Pr(v) = \Pr(E_v)$ the probability for $E_v$
(where the probability is over $x,a_1,\ldots,a_m$), and we denote
by $\P_{x|v} = \P_{x|E_v}$ the distribution of the random variable $x$ conditioned on the event~$E_v$.
Similarly,
for an edge~$e$ of the branching program $B$, let $E_e$ be
the event that $\T$ traverses the edge~$e$.
Denote, $\Pr(e) = \Pr(E_e)$, and
$\P_{x|e} = \P_{x|E_e}$.

A vertex $v$ of $B$ is called {\em significant} if
$$
\norm{\P_{x|v}}_{2} > 2^{\ell} \cdot 2^{-n}.
$$
Roughly speaking, this means that conditioning on the event that $\T$ reaches the
vertex~$v$, a non-negligible amount of information is known about $x$.
In order to guess $x$ with a non-negligible success probability, $\T$ must reach a significant vertex. Lemma~\ref{lemma-main1} shows that the probability that $\T$ reaches any significant vertex is negligible, and thus the main result follows.

To prove Lemma~\ref{lemma-main1}, we show that for every fixed significant vertex $s$, the probability that $\T$ reaches $s$ is at most
$2^{-\Omega(k \ell)}$ (which is smaller than one over the number of vertices
in~$B$). Hence, we can use a union bound to prove the lemma.

The proof  that the probability that $\T$ reaches $s$ is extremely small is the main
part of the proof.
To that end, we use the following functions to measure the progress made by the branching program towards reaching $s$.

Let $L_i$ be the set of vertices $v$ in layer-$i$ of $B$,
such that $\Pr (v) >0$. Let $\Gamma_i$ be the set of edges $e$ from layer-$(i-1)$ of $B$ to layer-$i$ of $B$,
such that $\Pr (e) >0$.
Let
$$
{\cal Z}_i =
\sum_{v \in L_i} \Pr(v) \cdot \langle \P_{x|v},\P_{x|s} \rangle^{k},
$$
$$
{\cal Z}'_i =
\sum_{e \in \Gamma_i} \Pr(e) \cdot \langle \P_{x|e},\P_{x|s} \rangle^{k}.
$$
We think of ${\cal Z}_i, {\cal Z}'_i$ as measuring the progress made by the branching program, towards reaching a state with distribution similar to
$\P_{x|s}$.

We show that each ${\cal Z}_i$ may only be negligibly larger than ${\cal Z}_{i-1}$. Hence, since it's easy to calculate that ${\cal Z}_0 = 2^{-2nk}$, it follows that
${\cal Z}_i$  is close to $2^{-2nk}$, for every $i$.
On the other hand, if $s$ is in layer-$i$ then
${\cal Z}_i$ is at least $\Pr(s) \cdot \langle \P_{x|s},\P_{x|s}\rangle^k$. Thus,
$\Pr(s) \cdot \langle \P_{x|s},\P_{x|s}\rangle^k$ cannot be much larger than
$2^{-2nk}$.
Since $s$ is significant,
$\langle \P_{x|s},\P_{x|s}\rangle^k > 2^{\ell k} \cdot 2^{-2nk}$
and hence $\Pr(s)$ is at most $2^{-\Omega(k \ell)}$.

The proof that ${\cal Z}_i$ may only be negligibly larger than ${\cal Z}_{i-1}$
is done in two steps:
Claim~\ref{claim-p2} shows by
a simple convexity argument that ${\cal Z}_i \leq {\cal Z}'_i$. The hard part, that is done in Claim~\ref{claim-p0} and Claim~\ref{claim-p1}, is to prove that
${\cal Z}'_i$ may only be negligibly larger than ${\cal Z}_{i-1}$.

For this proof, we
define for every vertex $v$, the set of edges
$\Gamma_{out}(v)$ that are going out of~$v$, such that $\Pr(e) >0$.
Claim~\ref{claim-p0} shows
that for every vertex $v$,
$$
\sum_{e \in \Gamma_{out}(v)} \Pr(e)
\cdot \langle \P_{x|e},\P_{x|s} \rangle^{k}$$
may only be negligibly higher than
$$
\Pr(v) \cdot
\langle \P_{x|v},\P_{x|s} \rangle^{k}.
$$

For the proof of
Claim~\ref{claim-p0}, which is the hardest proof in the paper,
and the most important place where our proof deviates from (and simplifies) the proof of~\cite{Raz17},
we consider the function $\P_{x|v} \cdot \P_{x|s}$. We first show how to bound
$\norm{\P_{x|v} \cdot \P_{x|s}}_2$. We then consider two cases:
If $\norm{\P_{x|v} \cdot \P_{x|s}}_1$
is negligible, then
$\langle \P_{x|v},\P_{x|s} \rangle^{k}$ is negligible and doesn't contribute much,
and we show that for every $e \in \Gamma_{out}(v)$,
$\langle \P_{x|e},\P_{x|s} \rangle^{k}$ is also negligible and doesn't contribute much.
 If $\norm{\P_{x|v} \cdot\P_{x|s}}_1$ is non-negligible,
we use the bound on $\norm{\P_{x|v} \cdot\P_{x|s}}_2$ and the
assumption that $M$ is a $(k,\ell)$-$L_2$-extractor
to show that for almost all edges $e \in \Gamma_{out}(v)$, we have that
$\langle \P_{x|e},\P_{x|s} \rangle^k$ is very close to
$\langle \P_{x|v},\P_{x|s} \rangle^k$.
Only an exponentially small ($2^{-k}$) fraction of edges are ``bad'' and give
a significantly larger $\langle \P_{x|e},\P_{x|s} \rangle^k$.

The reason that in the definitions of ${\cal Z}_i$ and ${\cal Z}'_i$
we raised $\langle \P_{x|v},\P_{x|s} \rangle$ and $\langle \P_{x|e},\P_{x|s} \rangle$ to the power of $k$ is that this is the largest power for which the contribution of the ``bad'' edges is still small (as their fraction is
$2^{-k}$).

This outline oversimplifies many details. Let us briefly mention two of them. First,
it is not so easy to bound $\norm{\P_{x|v} \cdot \P_{x|s}}_2$.
We do that by bounding $\norm{\P_{x|s}}_2$ and $\norm{\P_{x|v}}_{\infty}$.
In order to bound $\norm{\P_{x|s}}_2$,
we force $\T$ to stop whenever it reaches a significant vertex
(and thus we are able to bound
$\norm{\P_{x|v}}_2$ for every vertex reached by $\T$).
In order to bound
$\norm{\P_{x|v}}_{\infty}$, we force $\T$ to stop whenever $\P_{x|v}(x)$ is large, which allows us to consider
only the ``bounded'' part of $\P_{x|v}$. (This is related to the technique of {\em flattening} a distribution that was used in~\cite{KR}).
Second, some edges are so ``bad'' that
their contribution to ${\cal Z}'_i$ is huge so they cannot be ignored. We force $\T$ to stop before traversing any such edge.
(This is related to an idea that was used in~\cite{KRT} of analyzing separately paths that
traverse
``bad'' edges).
We show that the total probability that $\T$ stops before reaching a leaf is negligible.

\section{Main Result}
\begin{theorem} \label{thm:TM-main1}
Let $\tfrac{1}{100}< c <\tfrac{2}{3}$.
Fix $\gamma$
to be such that
$\tfrac{3c}{2} < \gamma^2 < 1$.

Let $\XX$, $\AA$ be two finite sets.
Let $n = \log_2|\XX|$.
Let $M: \AA \times \XX \rightarrow \{-1,1\}$ be a matrix
which is a $(k',\ell')$-$L_2$-extractor with error $2^{-r'}$,
for sufficiently large\footnote{By {\it ``sufficiently large''} we mean that $k',\ell',r'$ are larger than some  constant that depends on $\gamma$.}
$k',\ell'$ and $r'$, where $\ell' \leq n$.
Let
\begin{equation}
\label{eq:param setting}
\rr :=  \min\left\{ \tfrac{r'}{2}, \tfrac{(1-\gamma)k'}{2}, \tfrac{(1-\gamma)\ell'}{2} -1 \right\}.
\end{equation}

Let $B$ be a branching program of
length at most $2^{r}$ and width at most $2^{c \cdot k' \cdot \ell'}$
for the learning problem that corresponds to the matrix $M$.
Then,
the success probability of $B$
is at most $O(2^{-r})$.
\end{theorem}


\begin{proof}
Let
\begin{equation}
\label{eq:param setting2}
\kk := \gamma k'
\qquad \mbox{and} \qquad
\ell := \gamma \ell'/3.
\end{equation}
Note that by the assumption that $k',\ell'$ and $r'$ are sufficiently large, we get that $\kk, \ell$ and~$\rr$  are also sufficiently large.
Since $\ell' \le n$, we have
$\ell + \rr \le \tfrac{\gamma \ell'}{3} +\tfrac{(1-\gamma)\ell'}{2} < \tfrac{\ell'}{2} \le \tfrac{n}{2}$.
Thus,
\begin{equation}\label{eq:rr ell}\rr <n/2-\ell.\end{equation}

Let $B$ be a branching program of
length $m=2^{\rr}$ and width $d=2^{c \cdot k' \cdot \ell'}$
for the learning problem that corresponds to the matrix $M$.
We will show that the success probability of $B$
is at most $O(2^{-\rr})$.


\subsection{The Truncated-Path and Additional Definitions and Notation}

We will define the {\bf truncated-path}, $\T$, to be the same as the computation-path of $B$, except that it sometimes stops before reaching a leaf.
Formally,
we define  $\T$, together with several other definitions and notations, by induction on the layers of the branching program $B$.

Assume that we already defined the truncated-path $\T$, until it reaches layer-$i$ of $B$.
For a vertex $v$ in layer-$i$ of $B$, let $E_v$ be
the event that $\T$ reaches the vertex $v$.
For simplicity, we denote  by $\Pr(v) = \Pr(E_v)$ the probability for $E_v$
(where the probability is over $x,a_1,\ldots,a_m$), and we denote
by $\P_{x|v} = \P_{x|E_v}$ the distribution of the random variable $x$ conditioned on the event $E_v$.

There will be three cases in which the truncated-path $\T$ stops on a non-leaf $v$:
\begin{enumerate}
\item
If $v$ is a, so called, significant vertex, where the $\ell_2$ norm of $\P_{x|v}$ is non-negligible.
(Intuitively, this means that conditioned on the event that $\T$ reaches~$v$, a non-negligible amount of information is known about $x$).
\item
If $\P_{x|v} (x)$ is non-negligible.
(Intuitively, this means that conditioned on the event that $\T$ reaches~$v$, the correct element $x$ could have been guessed with a non-negligible probability).
\item
If $(M \cdot \P_{x|v}) (a_{i+1})$ is non-negligible.
(Intuitively, this means that
$\T$ is about to traverse a ``{bad}'' edge, which is traversed with a non-negligibly higher or lower probability than other edges).
\end{enumerate}

Next, we describe these three cases more formally.

\subsubsection*{Significant Vertices}

We say that a vertex $v$ in layer-$i$ of $B$ is {\bf significant} if
$$
\norm{\P_{x|v}}_{2} > 2^{\ell} \cdot 2^{-n}.
$$

\subsubsection*{Significant Values}

Even if $v$ is not significant, $\P_{x|v}$ may have relatively large values.
For a vertex $v$ in layer-$i$ of~$B$, denote by $\Sig(v)$ the set of all $x' \in \XX$, such that,
$$\P_{x|v}(x') > 2^{2\ell +2\rr} \cdot 2^{-n}.$$

\subsubsection*{Bad Edges}

For a vertex $v$ in layer-$i$ of $B$, denote by $\Bad(v)$ the set of all $\alpha \in \AA$, such that,
$$\left| (M \cdot\P_{x|v})(\alpha) \right|
\geq 2^{-r'}.$$

\subsubsection*{The Truncated-Path $\T$}

We define $\T$ by induction on the layers of the branching program $B$.
Assume that we already defined $\T$ until it reaches a vertex $v$ in layer-$i$ of~$B$.
The path $\T$ stops on $v$ if (at least) one of the following occurs:
\begin{enumerate}
\item
$v$ is significant.
\item
$x \in \Sig(v)$.
\item
$a_{i+1} \in \Bad(v)$.
\item
$v$ is a leaf.
\end{enumerate}
Otherwise, $\T$ proceeds by following the edge labeled by~$(a_{i+1}, b_{i+1})$
(same as the computational-path).

\subsection{Proof of Theorem~\ref{thm:TM-main1}}

Since $\T$ follows the computation-path of $B$, except that it sometimes stops before reaching a leaf, the success probability of $B$ is bounded (from above) by the probability that $\T$ stops before reaching a leaf, plus the probability that $\T$ reaches a leaf $v$ and $\tilde{x}(v) = x$.

The main lemma needed for the proof of Theorem~\ref{thm:TM-main1} is
Lemma~\ref{lemma-main1} that shows that
the probability that $\T$ reaches a significant
vertex is at most $O(2^{-\rr})$.

\begin{lemma} \label{lemma-main1}
The probability that $\T$ reaches a significant vertex is at most $O(2^{-\rr})$.
\end{lemma}

Lemma~\ref{lemma-main1} is proved in Section~\ref{section:mainlemma}. We will now show how the proof of Theorem~\ref{thm:TM-main1} follows from that lemma.

Lemma~\ref{lemma-main1}  shows that the probability that $\T$ stops on a non-leaf vertex, because of the first reason (i.e., that the vertex is significant), is small. The next two lemmas imply that the probabilities that $\T$ stops on a non-leaf vertex, because of the second and third reasons, are also small.

\begin{claim} \label{claim-A0}
If $v$ is a non-significant vertex of $B$ then
$$
\Pr_{x}
[x \in \Sig(v)  \; | \; E_v] \leq 2^{-2\rr}.
$$
\end{claim}

\begin{proof}
Since $v$ is not significant,
$$
\Ex_{x' \sim \P_{x|v}} \left[ \P_{x|v}(x') \right]  =
\sum_{x' \in \XX} \left[ \P_{x|v}(x')^{2} \right] =
2^n \cdot \Ex_{x' \in_R \XX} \left[ \P_{x|v}(x')^{2} \right]
\leq 2^{2\ell} \cdot 2^{-n}.
$$
Hence, by Markov's inequality,
$$
\Pr_{x' \sim \P_{x|v}}
\left[
\P_{x|v}(x') > 2^{2\rr} \cdot  2^{2\ell} \cdot 2^{-n}
\right]
\leq 2^{-2\rr}.
$$
Since conditioned on $E_v$, the distribution of $x$ is $\P_{x|v}$, we obtain
\[
\Pr_{x}
\left[x \in \Sig(v) \; \big| \; E_v\right] =
\Pr_{x}
\left[
\left(\P_{x|v}(x) > 2^{2\rr} \cdot  2^{2\ell} \cdot 2^{-n}\right)
\; \big|\; E_v\;
\right]
\leq 2^{-2\rr}.\qedhere
\]
\end{proof}

\begin{claim} \label{claim-A2}
If $v$ is a non-significant vertex of $B$ then
$$
\Pr_{a_{i+1}} [a_{i+1} \in \Bad(v)] \leq 2^{-2\rr}.
$$
\end{claim}

\begin{proof}
Since $v$ is not significant, $\norm{\P_{x|v}}_2 \le 2^{\ell} \cdot 2^{-n}$.
Since $\P_{x|v}$ is a distribution, $\norm{\P_{x|v}}_1 = 2^{-n}$.
Thus, $$\frac{\norm{\P_{x|v}}_2}{\norm{\P_{x|v}}_1} \le 2^{\ell} \le 2^{\ell'}.$$
Since $M$ is a $(k',\ell')$-$L_2$-extractor with error $2^{-r'}$, there are at most $2^{-k'} \cdot |A|$ elements $\alpha \in \AA$ with
$$
\left|\inner{M_{\alpha}, \P_{x|v}}\right| \ge 2^{-r'} \cdot {\norm{\P_{x|v}}_1} = 2^{-r'} \cdot 2^{-n}$$
The claim follows since $a_{i+1}$ is uniformly distributed over $\AA$ and since $k' \ge 2\rr$ (Equation~\eqref{eq:param setting}).
\end{proof}

We can now use Lemma~\ref{lemma-main1}, Claim~\ref{claim-A0} and Claim~\ref{claim-A2} to prove that the probability that~$\T$ stops before reaching a leaf is at most $O(2^{-\rr})$.
Lemma~\ref{lemma-main1}  shows that the probability that~$\T$ reaches a significant vertex and hence stops because of the first reason, is at most $O(2^{-\rr})$.
Assuming that $\T$ doesn't reach any significant vertex (in which case it would have stopped because of the first reason), Claim~\ref{claim-A0} shows that in each step, the probability that $\T$ stops because of the second reason, is at most $2^{-2\rr}$. Taking a union bound over the $m=2^{\rr}$ steps, the total probability that $\T$ stops because of the second reason, is at most $2^{-\rr}$. In the same way,
assuming that $\T$ doesn't reach any significant vertex (in which case it would have stopped because of the first reason), Claim~\ref{claim-A2} shows that in each step, the probability that $\T$ stops because of the third reason, is at most $2^{-2\rr}$. Again, taking a union bound over the $2^{\rr}$ steps, the total probability that $\T$ stops because of the third reason, is at most $2^{-\rr}$.
Thus, the total probability that~$\T$ stops (for any reason) before reaching a leaf is at most $O(2^{-\rr})$.

Recall that if $\T$ doesn't stop before reaching a leaf, it just follows the computation-path of~$B$.
Recall also that by
Lemma~\ref{lemma-main1},  the probability that $\T$ reaches a significant leaf is at most $O(2^{-\rr})$.
Thus, to bound (from above) the success probability of $B$ by
$O(2^{-\rr})$,
it remains to bound the probability that $\T$ reaches a non-significant leaf $v$
and
$\tilde{x}(v) = x$.
Claim~\ref{claim-A1} shows that for any non-significant leaf $v$, conditioned
on the event that $\T$ reaches~$v$, the probability
for $\tilde{x}(v) = x$ is at most $2^{- \rr}$, which completes the proof
of Theorem~\ref{thm:TM-main1}.

\begin{claim} \label{claim-A1}
If $v$ is a non-significant leaf of $B$ then
$$
\Pr [ \tilde{x}(v) = x \; | \; E_v]
\leq  2^{- \rr}.
$$

\end{claim}
\begin{proof}
Since $v$ is not significant,
$$
\Ex_{x' \in_R \XX} \left[ \P_{x|v}(x')^{2} \right]
\leq 2^{2\ell} \cdot 2^{-2n}.
$$
Hence, for every $x' \in \XX$,
$$
\Pr [x=x' \; | \; E_v]=
\P_{x|v}(x')
\leq 2^{\ell} \cdot 2^{-n/2}
\le 2^{- \rr}
$$
since $\rr \le n/2 - \ell$ (Equation~\eqref{eq:rr ell}).
In particular,
\[
\Pr [\tilde{x}(v) = x \; | \; E_v]
\le 2^{- \rr}.\qedhere
\]
\end{proof}

This completes the proof
of Theorem~\ref{thm:TM-main1}.
\end{proof}

\subsection{Proof of Lemma~\ref{lemma-main1}} \label{section:mainlemma}

\begin{proof}
We need to prove that the probability that $\T$ reaches any significant vertex is at most $O(2^{-\rr})$.
Let $s$ be a  significant vertex of $B$. We will bound from above
the probability that~$\T$ reaches~$s$, and then use
a union bound over all significant vertices of $B$.
Interestingly, the upper bound on the width of $B$ is used only in the union bound.

\subsubsection*{The Distributions $\P_{x|v}$ and  $\P_{x|e}$}

Recall that for a vertex $v$ of $B$, we denote by $E_v$
the event that $\T$ reaches the vertex $v$.
For simplicity, we denote  by $\Pr(v) = \Pr(E_v)$ the probability for $E_v$
(where the probability is over $x,a_1,\ldots,a_m$), and we denote
by $\P_{x|v} = \P_{x|E_v}$ the distribution of the random variable $x$ conditioned on the event $E_v$.

Similarly,
for an edge~$e$ of the branching program $B$, let $E_e$ be
the event that $\T$ traverses the edge~$e$.
Denote, $\Pr(e) = \Pr(E_e)$
(where the probability is over $x,a_1,\ldots,a_m$), and
$\P_{x|e} = \P_{x|E_e}$.

\begin{claim} \label{claim-d0}
For any edge~$e = (v,u)$ of $B$, labeled by $(a,b)$, such that
$\Pr(e) > 0$, for any $x' \in \XX$,
$$
\P_{x|e} (x')  = \left\{
\begin{array}{ccccc}
  0
  & \;\;\;\; \mbox{if } & x' \in \Sig(v) & \mbox{or} & M(a,x') \neq b \\
  \P_{x|v} (x') \cdot c_e^{-1}
  & \;\;\;\; \mbox{if } & x' \not \in \Sig(v)& \mbox{and} & M(a,x') = b
\end{array}
\right.
$$
where $c_e$ is a normalization factor that satisfies,
$$
c_e \geq
\tfrac{1}{2}
- 2\cdot 2^{-2\rr}.
$$
\end{claim}
\begin{proof}
Let $e = (v,u)$ be an edge of $B$, labeled by $(a,b)$, and such that
$\Pr(e) > 0$.
Since $\Pr(e) > 0$, the vertex $v$ is not significant (as otherwise $\T$ always stops on $v$ and hence $\Pr(e) = 0$).
Also, since $\Pr(e) > 0$, we know that $a \not \in \Bad(v)$
(as otherwise $\T$ never traverses~$e$ and hence $\Pr(e) = 0$).

If $\T$ reaches $v$, it  traverses the edge $e$ if and only if:
$x \not \in \Sig(v)$ (as otherwise $\T$ stops on~$v$)  and $M(a,x) = b$ and $a_{i+1} = a$.
Therefore, for any $x' \in \XX$,
$$
\P_{x|e} (x')  = \left\{
\begin{array}{ccccc}
  0
  & \;\;\;\; \mbox{if } & x' \in \Sig(v) & \mbox{or} & M(a,x') \neq b \\
  \P_{x|v} (x') \cdot c_e^{-1}
  & \;\;\;\; \mbox{if } & x' \not \in \Sig(v)& \mbox{and} & M(a,x') = b
\end{array}
\right.
$$
where $c_e$ is a normalization factor, given by
$$
c_e=
\sum_{\left\{ x' \; : \; x' \not \in \Sig(v) \; \wedge \; M(a,x') = b  \right\} }
\P_{x|v} (x') \; = \;
\Pr_x[(x \not \in \Sig(v)) \wedge  (M(a,x) = b) \; | \;E_v].
$$

Since $v$ is not significant, by Claim~\ref{claim-A0},
$$
\Pr_{x}
[x \in \Sig(v)  \; | \; E_v] \leq 2^{-2\rr}.
$$

Since $a \not \in \Bad(v)$,
$$
\left| \Pr_{x} [M(a,x) = 1  \; | \; E_v] -
\Pr_{x} [M(a,x) = -1  \; | \; E_v] \right|
=
\left| (M \cdot\P_{x|v})(a) \right|
\leq 2^{-r'},$$
and hence
$$
\Pr_{x}
[M(a,x) \neq b  \; | \; E_v] \leq \tfrac{1}{2} + 2^{-r'}.
$$
Hence, by the union bound,
$$
c_e=
\Pr_x[(x \not \in \Sig(v)) \wedge  (M(a,x) = b) \; | \;E_v] \geq
\tfrac{1}{2}
- 2^{-r'} - 2^{-2\rr}
\geq
\tfrac{1}{2}
- 2\cdot 2^{-2\rr}
$$
(where the last inequality follows since
$\rr \le r'/2$, by Equation~\eqref{eq:param setting}).
\end{proof}

\subsubsection*{Bounding the Norm of $\P_{x|s}$}

We will show that $\norm{\P_{x|s}}_{2}$ cannot be too large. Towards this, we will first prove that for every edge $e$ of $B$ that is traversed by $\T$ with probability larger than zero, $\norm{\P_{x|e}}_{2}$ cannot be too large.

\begin{claim} \label{claim-b0}
For any edge~$e$ of $B$, such that
$\Pr(e) > 0$,
$$ \norm{\P_{x|e}}_{2} \leq 4 \cdot 2^{\ell} \cdot 2^{-n}.$$
\end{claim}
\begin{proof}
Let $e = (v,u)$ be an edge of $B$, labeled by $(a,b)$, and such that
$\Pr(e) > 0$.
Since $\Pr(e) > 0$, the vertex $v$ is not significant (as otherwise $\T$ always stops on $v$ and hence $\Pr(e) = 0$).
Thus,
$$
\norm{\P_{x|v}}_{2} \leq 2^{\ell} \cdot 2^{-n}.
$$

By Claim~\ref{claim-d0}, for any $x' \in \XX$,
$$
\P_{x|e} (x')  = \left\{
\begin{array}{ccccc}
  0
  & \;\;\;\; \mbox{if } & x' \in \Sig(v) & \mbox{or} & M(a,x') \neq b \\
  \P_{x|v} (x') \cdot c_e^{-1}
  & \;\;\;\; \mbox{if } & x' \not \in \Sig(v)& \mbox{and} & M(a,x') = b
\end{array}
\right.
$$
where $c_e$ satisfies,
$$
c_e \geq
\tfrac{1}{2}
- 2\cdot 2^{-2\rr} > \tfrac{1}{4}
$$
(where the last inequality holds because we assume that $k',\ell',r'$ and thus $\rr$ are sufficiently large.)
Thus,
\[
\norm{\P_{x|e}}_{2} \leq c_e^{-1} \cdot \norm{\P_{x|v}}_{2} \leq
4 \cdot 2^{\ell} \cdot 2^{-n}\qedhere
\]
\end{proof}

\begin{claim} \label{claim-b1}
$$ \norm{\P_{x|s}}_{2} \leq 4 \cdot 2^{\ell} \cdot 2^{-n}.$$
\end{claim}
\begin{proof}
Let $\Gamma_{in}(s)$ be the set of all edges $e$ of $B$, that are going into $s$, such that $\Pr(e) >0$.
Note that $$\sum_{e \in \Gamma_{in}(s)} \Pr(e) = \Pr(s).$$

By the law of total probability,
for every $x' \in \XX$,
$$
\P_{x|s} (x') =
\sum_{e \in \Gamma_{in}(s)} \tfrac{\Pr(e)}{\Pr(s)} \cdot \P_{x|e} (x'),
$$
and hence by Jensen's inequality,
$$
\P_{x|s} (x')^2 \leq
\sum_{e \in \Gamma_{in}(s)} \tfrac{\Pr(e)}{\Pr(s)} \cdot \P_{x|e} (x')^2.
$$
Summing over $x' \in \XX$, we obtain,
$$
\norm{\P_{x|s}}_{2}^2 \leq
\sum_{e \in \Gamma_{in}(s)} \tfrac{\Pr(e)}{\Pr(s)} \cdot
\norm{\P_{x|e}}_{2}^2.
$$

By Claim~\ref{claim-b0},
for any $e \in \Gamma_{in}(s)$,
$$ \norm{\P_{x|e}}_{2}^2 \leq \left( 4 \cdot 2^{\ell} \cdot 2^{-n} \right)^2.$$
Hence,
\[
\norm{\P_{x|s}}_{2}^2 \leq \left( 4 \cdot 2^{\ell} \cdot 2^{-n} \right)^2.\qedhere
\]
\end{proof}

\subsubsection*{Similarity to a Target Distribution}

Recall that for two functions
$f,g: \XX \rightarrow \Reals^+$, we defined
$$
\langle f,g \rangle = \Ex_{z \in_R \XX} [ f(z) \cdot g(z) ].
$$
We think of $\langle f,g \rangle$ as a measure for the similarity between a function $f$ and a target function~$g$.
Typically $f,g$ will be distributions.

\begin{claim} \label{claim-s1}
$$\langle \P_{x|s},\P_{x|s} \rangle > 2^{2\ell} \cdot 2^{-2n}.$$
\end{claim}
\begin{proof}
Since $s$ is significant,
\[
\langle \P_{x|s},\P_{x|s} \rangle =
\norm{\P_{x|s}}^2_{2} > 2^{2\ell} \cdot 2^{-2n}.\qedhere
\]
\end{proof}

\begin{claim} \label{claim-s2}
$$\langle \U_X,\P_{x|s} \rangle = 2^{-2n},$$
where $\U_X$ is the uniform distribution over $\XX$.
\end{claim}
\begin{proof}
Since $\P_{x|s}$ is a distribution,
\[\langle \U_X,\P_{x|s} \rangle =  2^{-2n} \cdot \sum_{z \in \XX} \P_{x|s}(z) = 2^{-2n}.\qedhere\]
\end{proof}

\subsubsection*{Measuring the Progress}

For $i \in \{0,\ldots ,m\}$, let $L_i$ be the set of vertices $v$ in layer-$i$ of $B$,
such that $\Pr (v) >0$. For $i \in \{1,\ldots ,m\}$,
let $\Gamma_i$ be the set of edges $e$ from layer-$(i-1)$ of $B$ to layer-$i$ of $B$,
such that $\Pr (e) >0$.
Recall that $\kk = \gamma k'$ (Equation~\eqref{eq:param setting2}).

For $i \in \{0,\ldots ,m\}$, let
$$
{\cal Z}_i =
\sum_{v \in L_i} \Pr(v) \cdot \langle \P_{x|v},\P_{x|s} \rangle^{\kk}.
$$
For $i \in \{1,\ldots ,m\}$, let
$$
{\cal Z}'_i =
\sum_{e \in \Gamma_i} \Pr(e) \cdot \langle \P_{x|e},\P_{x|s} \rangle^{\kk}.
$$

We think of ${\cal Z}_i, {\cal Z}'_i$ as measuring the progress made by the branching program, towards reaching a state with distribution similar to
$\P_{x|s}$.

For a vertex $v$ of $B$, let
$\Gamma_{out}(v)$ be the set of all edges $e$ of $B$, that are going out of $v$, such that $\Pr(e) >0$.
Note that $$\sum_{e \in \Gamma_{out}(v)} \Pr(e) \leq \Pr(v).$$
(We don't always have an equality here, since sometimes $\T$ stops on $v$).

The next four claims show that the progress made by the branching program is slow.

\begin{claim} \label{claim-p0}
For every vertex $v$ of $B$, such that $\Pr (v) > 0$,
$$
\sum_{e \in \Gamma_{out}(v)} \tfrac{\Pr(e)}{\Pr(v)}
\cdot \langle \P_{x|e},\P_{x|s} \rangle^{\kk}
\leq
\langle \P_{x|v},\P_{x|s} \rangle^{\kk} \cdot
\left( 1 + 2^{-\rr}\right)^{\kk}
+ \left( 2^{-2n +2} \right)^{\kk}.
$$
\end{claim}
\begin{proof}
If $v$ is significant or $v$ is a leaf, then $\T$ always stops on $v$ and hence
$\Gamma_{out}(v)$ is empty and thus the left hand side is equal to zero and the right hand side is positive, so the claim follows trivially.
Thus, we can assume that $v$ is not significant and is not a leaf.

Define $P: \XX \rightarrow \Reals^+$ as follows.
For any $x' \in \XX$,
$$
P (x')  = \left\{
\begin{array}{ccc}
  0
  & \;\;\;\; \mbox{if } & x' \in \Sig(v)  \\
  \P_{x|v} (x')
  & \;\;\;\; \mbox{if } & x' \not \in \Sig(v)
\end{array}
\right.
$$
Note that by the definition of
$\Sig(v)$,  for any $x' \in \XX$,
\begin{equation} \label{e11}
P(x') \leq 2^{2\ell + 2\rr} \cdot 2^{-n}.
\end{equation}

Define $f: \XX \rightarrow \Reals^+$ as follows.
For any $x' \in \XX$,
$$
f(x')  = P(x') \cdot  \P_{x|s}(x').
$$
By Claim~\ref{claim-b1} and Equation~\eqref{e11},
\begin{equation} \label{e12}
\norm{f}_{2} \leq
2^{2\ell + 2\rr} \cdot 2^{-n} \cdot
\norm{\P_{x|s}}_{2} \leq
2^{2\ell + 2\rr}
\cdot 2^{-n} \cdot 4 \cdot 2^{\ell} \cdot 2^{-n}
= 2^{3\ell + 2\rr +2}
\cdot 2^{-2n}.
\end{equation}

By Claim~\ref{claim-d0},
for any edge~$e \in \Gamma_{out}(v)$, labeled by $(a,b)$, for any $x' \in \XX$,
$$
\P_{x|e} (x')  = \left\{
\begin{array}{ccc}
  0
  & \;\;\;\; \mbox{if } & M(a,x') \neq b \\
  P(x') \cdot c_e^{-1}
  & \;\;\;\; \mbox{if } & M(a,x') = b
\end{array}
\right.
$$
where $c_e$ satisfies,
$$
c_e \geq
\tfrac{1}{2}
- 2\cdot 2^{-2\rr}.
$$
Therefore,
for any edge~$e \in \Gamma_{out}(v)$, labeled by $(a,b)$, for any $x' \in \XX$,
$$
\P_{x|e} (x') \cdot \P_{x|s} (x') = \left\{
\begin{array}{ccc}
  0
  & \;\;\;\; \mbox{if } & M(a,x') \neq b \\
  f(x') \cdot c_e^{-1}
  & \;\;\;\; \mbox{if } & M(a,x') = b
\end{array}
\right.
$$
and hence, we have
\begin{align}
\nonumber
\langle \P_{x|e},\P_{x|s} \rangle &=
\Ex_{x' \in_R \XX} [ \P_{x|e}(x') \cdot \P_{x|s}(x') ]
=
\Ex_{x' \in_R \XX}
[f(x') \cdot c_e^{-1} \cdot \mathbf{1}_{\{x' \in X \; : \; M(a,x')=b\}}]
\\
\nonumber
&=
\Ex_{x' \in_R \XX}
\left[f(x') \cdot c_e^{-1} \cdot \tfrac{(1+b\cdot M(a,x'))}{2}\right]
=
\left(\norm{f}_1 + b \cdot \inner{M_a, f}\right) \cdot (2c_e)^{-1}
\\
\label{e13}
&<
\left( \norm{f}_1+ |\inner{M_a, f}| \right)
\cdot \left(1 + 2^{-2\rr + 3}\right)
\end{align}
(where the last inequality holds by the bound that we have on $c_e$, because we assume that $k',\ell',r'$ and thus $\rr$ are sufficiently large).

We will now consider two cases:

\subsubsection*{Case I: $\norm{f}_1 <  2^{-2n}$}
In this case, we bound  $|\inner{M_a, f}| \leq \norm{f}_1$
(since $f$ is non-negative and the entries of $M$ are in~$\{-1,1\}$)
and
$(1 + 2^{-2\rr +3}) < 2$ (since we assume that $k',\ell',r'$ and thus $\rr$ are sufficiently large) and obtain
for any edge~$e \in \Gamma_{out}(v)$,
$$\langle \P_{x|e},\P_{x|s} \rangle
< 4  \cdot 2^{-2n}.$$
Since $\sum_{e \in \Gamma_{out}(v)} \tfrac{\Pr(e)}{\Pr(v)} \leq 1$,
Claim~\ref{claim-p0} follows, as the left hand side of the claim is smaller than the second term on the right hand side.

\subsubsection*{Case II: $\norm{f}_1 \geq 2^{-2n}$}
For every $a \in \AA$, define
$$t(a) = \frac{|\inner{M_a, f}|}{\norm{f}_1}.$$
By Equation~\eqref{e13},
\begin{equation}\label{e14}
\langle \P_{x|e},\P_{x|s} \rangle^{\kk} <
\norm{f}_1^{\kk} \cdot
\left( 1 + t(a) \right)^{\kk}
\cdot \left( 1 + 2^{-2\rr + 3}\right)^{\kk}.
\end{equation}

Note that by the definitions of $P$ and $f$,
$$
\norm{f}_1 = \Ex_{x' \in_R \XX} [f(x')] =
\langle P,\P_{x|s} \rangle \leq \langle \P_{x|v},\P_{x|s} \rangle.
$$
Note also that for every $a \in \AA$, there is at most one edge
$e_{(a,1)} \in \Gamma_{out}(v)$, labeled by $(a,1)$, and
at most one edge
$e_{(a,-1)} \in \Gamma_{out}(v)$, labeled by $(a,-1)$,
and we have
$$\tfrac{\Pr(e_{(a,1)})}{\Pr(v)} + \tfrac{\Pr(e_{(a,-1)})}{\Pr(v)}
\leq \tfrac{1}{|A|},$$
since $\tfrac{1}{|A|}$ is the probability that the next sample read by the program is $a$.
%
Thus, summing over all $e \in \Gamma_{out}(v)$, by Equation~\eqref{e14},
\begin{equation}\label{e15}
\sum_{e \in \Gamma_{out}(v)}
\tfrac{\Pr(e)}{\Pr(v)}  \cdot \langle \P_{x|e},\P_{x|s} \rangle^{\kk} <
\langle \P_{x|v},\P_{x|s} \rangle ^{\kk} \cdot
\Ex_{a \in_R \AA} \left[ \left( 1 + t(a) \right)^{\kk} \right]
\cdot \left( 1 + 2^{-2\rr +3}\right)^{\kk}.
\end{equation}

It remains to bound
\begin{equation} \label{e16}
\Ex_{a \in_R \AA} \left[ \left( 1 + t(a) \right)^{\kk} \right],
\end{equation}
using the properties of the matrix $M$ and the bounds on the $\ell_2$ versus $\ell_1$ norms of $f$.

By Equation~\eqref{e12}, the assumption that $\norm{f}_1 \ge 2^{-2n}$, Equation~\eqref{eq:param setting} and Equation~\eqref{eq:param setting2}, we get
\begin{equation*} 
\frac{\norm{f}_{2}}{\norm{f}_1} \leq
2^{3\ell + 2\rr +2} \le 2^{\ell'}\;.
\end{equation*}
Since $M$ is a $(k',\ell')$-$L_2$-extractor with error $2^{-r'}$,
 there are at most $2^{-k'} \cdot |A|$ rows $a\in \AA$ with
$t(a)  = \frac{|\inner{M_a,f}|}{\norm{f}_1} \ge 2^{-r'}$.
We bound the expectation in Equation~\eqref{e16}, by splitting
the expectation into two sums
\begin{equation}
	\label{e18}
\Ex_{a \in_R \AA} \left[ \left( 1 + {t(a)} \right)^{\kk} \right]
= \tfrac{1}{|A|} \cdot  \sum_{a \;: \; t(a) \leq 2^{-r'}}
\left( 1 + t(a) \right)^{\kk}
+
\tfrac{1}{|A|} \cdot  \sum_{a \;: \; t(a) > 2^{-r'}}
\left( 1 + t(a) \right)^{\kk}.
\end{equation}

We bound the first sum in Equation~\eqref{e18} by $(1+2^{-r'})^{\kk}$.
As for the second sum in Equation~\eqref{e18}, we
know that it is a sum of at most $2^{-k'} \cdot |\AA|$ elements,
and since for every $a \in \AA$, we have $t(a) \leq 1$, we have
\begin{equation*} 
\tfrac{1}{|A|} \cdot  \sum_{a \;: \; t(a) > 2^{-r'}}
\left( 1 + t(a) \right)^{\kk} \leq
2^{-k'} \cdot 2^{\kk}
\le
2^{-2\rr}\;
\end{equation*}
(where in the last inequality we used Equations~\eqref{eq:param setting} and~\eqref{eq:param setting2}).
Overall, using Equation~\eqref{eq:param setting} again, we get
\begin{equation}\label{e18.5}
\Ex_{a \in_R \AA} \left[ \left( 1 + t(a) \right)^{\kk} \right] \le (1+2^{-r'})^{\kk} + 2^{-2\rr} \le (1+2^{-2\rr})^{\kk+1}.
\end{equation}
Substituting
Equation~\eqref{e18.5} into Equation~\eqref{e15}, we obtain
\begin{align*}
\sum_{e \in \Gamma_{out}(v)}
\tfrac{\Pr(e)}{\Pr(v)}  \cdot \langle \P_{x|e},\P_{x|s} \rangle^{\kk} &<
\langle \P_{x|v},\P_{x|s} \rangle ^{\kk}
\cdot  \left(1 + 2^{-2\rr}\right)^{\kk+1}
\cdot \left( 1 + 2^{-2\rr +3}\right)^{\kk}
\\
&<
\langle \P_{x|v},\P_{x|s} \rangle ^{\kk}
\cdot \left( 1 + 2^{-\rr}\right)^{\kk}
\end{align*}
(where the last inequality uses the assumption that $\rr$ is sufficiently large).
This completes the proof of Claim~\ref{claim-p0}.
\end{proof}

\begin{claim} \label{claim-p1}
For every $i \in \{1,\ldots ,m\}$,
$${\cal Z}'_i \leq  {\cal Z}_{i-1}
\cdot
\left( 1 + 2^{-\rr}\right)^{\kk}
+ \left( 2^{-2n +2} \right)^{\kk}.
$$\end{claim}
\begin{proof}
By Claim~\ref{claim-p0},
\begin{align*}
{\cal Z}'_i =
\sum_{e \in \Gamma_i} \Pr(e) \cdot \langle \P_{x|e},\P_{x|s} \rangle^{\kk}
&=
\sum_{v \in L_{i-1}} \Pr(v) \cdot
\sum_{e \in \Gamma_{out}(v)} \tfrac{\Pr(e)}{\Pr(v)}
\cdot \langle \P_{x|e},\P_{x|s} \rangle^{\kk}
\\&\leq
\sum_{v \in L_{i-1}} \Pr(v) \cdot
\left(
\langle \P_{x|v},\P_{x|s} \rangle^{\kk} \cdot
\left( 1 + 2^{-\rr}\right)^{\kk}
+ \left( 2^{-2n +2} \right)^{\kk}
\right)
\\&=
{\cal Z}_{i-1} \cdot
\left( 1 + 2^{-\rr}\right)^{\kk} +
\sum_{v \in L_{i-1}} \Pr(v) \cdot
\left( 2^{-2n +2} \right)^{\kk}
\\&\leq
{\cal Z}_{i-1} \cdot
\left( 1 + 2^{-\rr}\right)^{\kk} +
\left( 2^{-2n +2} \right)^{\kk}\qedhere
\end{align*}
\end{proof}

\begin{claim} \label{claim-p2}
For every $i \in \{1,\ldots ,m\}$,
$$
{\cal Z}_{i} \leq
{\cal Z}'_i.
$$
\end{claim}
\begin{proof}
For any $v \in L_i$,
let $\Gamma_{in}(v)$ be the set of all edges $e \in \Gamma_i$, that are going into $v$.
Note that $$\sum_{e \in \Gamma_{in}(v)} \Pr(e) = \Pr(v).$$

By the law of total probability,
for every $v \in L_i$ and every $x' \in \XX$,
$$
\P_{x|v} (x') =
\sum_{e \in \Gamma_{in}(v)} \tfrac{\Pr(e)}{\Pr(v)} \cdot \P_{x|e} (x'),
$$
and hence
$$
\langle \P_{x|v},\P_{x|s} \rangle =
\sum_{e \in \Gamma_{in}(v)} \tfrac{\Pr(e)}{\Pr(v)} \cdot
\langle \P_{x|e},\P_{x|s} \rangle.
$$
Thus,
by Jensen's inequality,
$$
\langle \P_{x|v},\P_{x|s} \rangle  ^{\kk}
\leq
\sum_{e \in \Gamma_{in}(v)} \tfrac{\Pr(e)}{\Pr(v)} \cdot
\langle \P_{x|e},\P_{x|s} \rangle ^{\kk}.
$$

Summing over all $v \in L_i$, we get
$$
{\cal Z}_{i} =
\sum_{v \in L_{i}} \Pr(v) \cdot
\langle \P_{x|v},\P_{x|s} \rangle  ^{\kk}
\leq
\sum_{v \in L_{i}} \Pr(v) \cdot
\sum_{e \in \Gamma_{in}(v)} \tfrac{\Pr(e)}{\Pr(v)} \cdot
\langle \P_{x|e},\P_{x|s} \rangle ^{\kk}
$$
\[
=
\sum_{e \in \Gamma_i} \Pr(e) \cdot \langle \P_{x|e},\P_{x|s} \rangle^{\kk}
=
{\cal Z}'_i.\qedhere
\]
\end{proof}

\begin{claim} \label{claim-p3}
For every $i \in \{1,\ldots ,m\}$,
$$
{\cal Z}_i \leq
2^{4\kk + 2\rr} \cdot 2^{-2\kk \cdot n}.
$$
\end{claim}

\begin{proof}
By Claim~\ref{claim-s2}, ${\cal Z}_0 = (2^{-2n})^{\kk}$.
By Claim~\ref{claim-p1} and Claim~\ref{claim-p2}, for every
$i \in \{1,\ldots ,m\}$,
$$
{\cal Z}_i \leq  {\cal Z}_{i-1}
\cdot
\left( 1 + 2^{-\rr}\right)^{\kk}
+ \left( 2^{-2n +2} \right)^{\kk}.
$$
Hence, for every
$i \in \{1,\ldots ,m\}$,
$$
{\cal Z}_i \leq  \left( 2^{-2n +2} \right)^{\kk}
\cdot
m \cdot
\left( 1 + 2^{-\rr}\right)^{\kk m}.
$$
Since $m = 2^{\rr}$,
\[
{\cal Z}_i \leq
2^{-2\kk \cdot n} \cdot 2^{2\kk} \cdot 2^{\rr} \cdot e^{\kk}
\leq
2^{-2\kk \cdot n} \cdot 2^{4\kk + 2\rr}.\qedhere\]
\end{proof}

\subsubsection*{Proof of Lemma~\ref{lemma-main1}}

We can now complete the proof of Lemma~\ref{lemma-main1}.
Assume that $s$ is in layer-$i$ of $B$.
By Claim~\ref{claim-s1},
$${\cal Z}_i \geq \Pr(s) \cdot \langle \P_{x|s},\P_{x|s} \rangle ^{\kk}
> \Pr(s) \cdot \left( 2^{2\ell} \cdot 2^{-2n} \right)^{\kk}
= \Pr(s) \cdot 2^{2\ell \cdot \kk} \cdot 2^{-2\kk \cdot n}.$$
On the other hand, by Claim~\ref{claim-p3},
$$
{\cal Z}_i \leq  2^{4\kk + 2\rr} \cdot 2^{-2\kk \cdot n}.
$$
Thus, using Equation~\eqref{eq:param setting} and Equation~\eqref{eq:param setting2}, we get
$$
\Pr(s) \leq
2^{4\kk + 2\rr} \cdot
2^{-2\ell \cdot \kk}
\le 2^{4 k'} \cdot 2^{-(2 \gamma^2/3) \cdot (k'\ell')}.
$$


Recall that we assumed that the width of $B$ is at most $2^{c k' \ell'}$ for some constant $c<2/3$,
and that the length of $B$ is at most $2^{\rr}$.
Recall that we fixed
$\gamma$ such that $2\gamma^2/3 > c$.
Taking a union bound over at most $2^{\rr} \cdot 2^{c k' \ell'} \le 2^{k'} \cdot 2^{c k' \ell'}$ significant vertices of $B$, we conclude that the probability that $\T$ reaches any significant vertex is at most $2^{-\Omega(k' \ell')}$.
Since we assume that $k'$ and $\ell'$ are sufficiently large,
$2^{-\Omega(k' \ell')}$
is certainly at most $2^{-k'}$, which is at most $2^{-\rr}$.
\end{proof}

\subsection{Lower Bounds for Weak Learning}
In this section, we show that under the same conditions of Theorem~\ref{thm:TM-main1}, the branching program cannot even weakly-learn the function. That is, we show that the branching program cannot output a hypothesis $h : \AA \to \{-1,1\}$ with a non-negligible correlation with the function defined by the true unknown $x$.
We change the definition of the branching program and associate with each leaf $v$ a hypothesis $h_v: \AA \to \{-1,1\}$.
We measure the success as the correlation between $h_v$ and the function defined by the true unknown $x$.

Formally, for any $x\in \XX$, let $M^{(x)}: \AA \to \{-1,1\}$ be the function corresponding to the $x$-th column of $M$.
We define the {\sf value} of the program as $\Ex\left[\left|\inner{h_v, M^{(x)}}\right|\right]$, where the expectation is over $x,a_1,\ldots,a_m$
(recall that $x$ is uniformly distributed over $\XX$ and $a_1,\ldots,a_m$ are uniformly distributed over $\AA$, and for every $t$, $b_t = M(a_t,x)$). The following claim bounds the expected correlation between $h_v$ and $M^{(x)}$, conditioned on reaching a non-significant leaf.

\begin{claim} \label{claim:low-correlation}
If $v$ is a non-significant leaf, then
$$
\Ex_{x}\Big[ \left|\inner{h_v, M^{(x)}}\right| \;\Big|\; E_v\Big] \le O(2^{-r/2}).
$$
\end{claim}

\begin{proof}
We expand the expected correlation between $h_v$ and $M^{(x)}$, squared:
\begin{align*}
\Ex_{x}\Big[\left|\inner{h_v,M^{(x)}}\right|\;\Big|\;E_v\Big]^2
&\le \Ex_{x}\Big[\inner{h_v,M^{(x)}}^2\;\Big|\;E_v\Big] = \sum_{x' \in \XX}{\P_{x|v}(x') \cdot \inner{h_v,M^{(x')}}^2}\\
&= \sum_{x'\in \XX}  \P_{x|v}(x')\cdot \Ex_{a,a' \in_R \AA}
[h_v(a) \cdot M(a,x') \cdot h_v(a') \cdot M(a',x')]\\
&= \Ex_{a,a'\in_R \AA} \bigg[h_v(a) \cdot h_v(a') \cdot \sum_{x' \in \XX} \P_{x|v}(x')\cdot M(a,x') \cdot M(a',x')\bigg]\\
&\le \Ex_{a,a'\in_R \AA} \left[\left|\sum_{x' \in \XX} \P_{x|v}(x')\cdot M(a,x') \cdot M(a',x')\right|\right]\\
&= \Ex_{a\in_R \AA} \left[\Ex_{a' \in_R \AA}\left[\left|\sum_{x' \in \XX} \P_{x|v}(x')\cdot M(a,x') \cdot M(a',x')\right|\right]\right]\;.
\end{align*}

Next, we show that $\Ex_{a' \in_R \AA}\left[\left|\sum_{x' \in \XX} \P_{x|v}(x')\cdot M(a,x') \cdot M(a',x')\right|\right]\ \le 4\cdot 2^{-r}$ for any $a\in \AA$.
Fix $a \in \AA$.
Let $q_{a}: \XX \to \Reals$ be the function
defined by $q_{a}(x') = \P_{x|v}(x') \cdot M(a,x')$ for $x' \in \XX$.
Since $|q_a(x')|=|\P_{x|v}(x')|$ for any $x' \in \XX$  and since $v$ is a non-significant vertex, we get
$$\norm{q_a}_2 = \norm{\P_{x|v}}_2 \le 2^{\ell} \cdot 2^{-n} \qquad\text{and}\qquad \norm{q_a}_1 = \norm{\P_{x|v}}_1 = 2^{-n}.$$
Hence, $\frac{\norm{q_a}_2}{\norm{q_a}_1}\le 2^{\ell}$.
We would like to use the fact that $M$ is a $(k',\ell')$-$L_2$-extractor with error $2^{-r'}$ to show that there aren't many rows of $M$ with a large inner product with $q_a$. However, $q_a$ can get negative values and the definition of $L_2$-extractors only handles non-negative functions $f: \XX \to \Reals^{+}$.
To solve this issue, we use the following lemma, proved in Section~\ref{sec:useful}.

\begin{lemma}\label{lem:negative}
Suppose that $M: \AA \times \XX \to \{-1,1\}$ is a $(k',\ell')$-$L_2$-extractor with error at most~$2^{-r}$.
Let $f: \XX \to \Reals$ be any function (i.e., $f$ can get negative values) with $\frac{\norm{f}_2}{\norm{f}_1} \le 2^{\ell'-r}$.
Then, there are at most $2\cdot 2^{-k'}\cdot |A|$ rows $a\in A$ with $\frac{|\inner{M_a, f}|}{\norm{f}_1} \ge 2\cdot 2^{-r}$.
\end{lemma}
Since
$M$ is a $(k',\ell')$-$L_2$-extractor with error at most~$2^{-r'}$, and since
$r < r'$, we have that $M$ is also a $(k',\ell')$-$L_2$-extractor with error at most~$2^{-r}$.
Since  $\frac{\norm{q_a}_2}{\norm{q_a}_1}\le 2^{\ell} \le 2^{\ell'-r}$,
we can apply Lemma~\ref{lem:negative} with $f = q_a$, and error $2^{-r}$.
We get that there are at most $2\cdot 2^{-k'} \cdot |A|$ rows $a' \in \AA$ with
$\frac{\left|\inner{q_a, M_{a'}}\right|}{\norm{q_a}_1} \ge 2\cdot 2^{-r}$.
Thus,
$$
\Ex_{a' \in_R \AA}
\left[
\left|\sum_{x' \in \XX} q_{a}(x') \cdot M(a',x')\right|
\right]
= \Ex_{a' \in_R \AA}
\left[
\frac{\left|\inner{q_a, M_{a'}}\right|}{\norm{q_a}_1}
\right] \le 2\cdot 2^{-k'} + 2\cdot 2^{-r} \le 4\cdot 2^{-\rr}\;.
$$

Overall, we get that  $\Ex_{x}\big[|\inner{h_v,M^{(x)}}|\;\big|\;E_v\big]^2 \le 4\cdot 2^{-\rr}$. Taking square roots of both sides of the last inequality completes the proof.
\end{proof}

Lemma~\ref{lemma-main1}, Claim~\ref{claim-A0} and Claim~\ref{claim-A2} show that the probability that~$\T$ stops before reaching a leaf is at most $O(2^{-\rr})$.
Combining this with Claim~\ref{claim:low-correlation} we get that (under the same conditions of Theorem~\ref{thm:TM-main1})
$$\Ex[ \left|\inner{h_v, M^{(x)}}\right|] \le \Pr[\T\text{~stops}] + O(2^{-r/2}) \le O(2^{-r/2}),$$
where the expectation and probability are taken over $x\in_R \XX$ and $a_1, \ldots,a_m\in_R \AA$.
We get the following theorem as a conclusion.

\begin{theorem} \label{thm:TM-main2}
Let $\tfrac{1}{100}< c <\tfrac{2}{3}$.
Fix $\gamma$ to be such that $\tfrac{3c}{2} < \gamma^2 < 1$.

Let $\XX$, $\AA$ be two finite sets.
Let $n = \log_2|\XX|$.
Let $M: \AA \times \XX \rightarrow \{-1,1\}$ be a matrix
which is a $(k',\ell')$-$L_2$-extractor with error $2^{-r'}$,
for sufficiently large\footnote{By {\it ``sufficiently large''} we mean that $k',\ell',r'$ are larger than some  constant that depends on $\gamma$.}
$k',\ell'$ and $r'$, where $\ell' \leq n$.
Let
\begin{equation*}
\rr :=  \min\left\{ \tfrac{r'}{2}, \tfrac{(1-\gamma)k'}{2}, \tfrac{(1-\gamma)\ell'}{2} -1 \right\}.
\end{equation*}

Let $B$ be a branching program of
length at most $2^{r}$ and width at most $2^{c \cdot k' \cdot \ell'}$
for the learning problem that corresponds to the matrix $M$.
Then, $$\Ex[ \left| \inner{h_v, M^{(x)}}\right|] \le O(2^{-r/2})\;.$$
\end{theorem}

In particular, the probability that the hypothesis agrees with the  function defined by the true unknown $x$, on more than $1/2 +  2^{-r/4}$ of the inputs, is at most $O(2^{-r/4})$.

\subsection{Main Corollary}

\begin{corollary} \label{cor:main1}
There exists a (sufficiently small) constant $c >0$, such that:

Let $\XX$, $\AA$ be two finite sets.
Let $M: \AA \times \XX \rightarrow \{-1,1\}$ be a matrix.
Assume that $k,\ell, r \in \N$ are such that any submatrix of $M$ of at least $2^{-k} \cdot |A|$ rows and at least
$2^{-\ell} \cdot |X|$ columns, has a bias of at most $2^{-r}$.


Let $B$ be a branching program of
length at most $2^{c \cdot r}$ and width at most $2^{c \cdot k \cdot \ell}$
for the learning problem that corresponds to the matrix $M$.
Then,
the success probability of $B$
is at most $2^{-\Omega(r)}$.
\end{corollary}

\begin{proof}
By Lemma~\ref{lemma:error-minent} (stated and proved below),
there exist
$k' = k + \Omega(r)$,  $\; \ell' = \ell+ \Omega(r)$, and $ r' = \Omega(r)$,
such  that:
any submatrix of $M$ of at least $2^{-k'} \cdot |A|$ rows and at least
$2^{-\ell'} \cdot |X|$ columns, has a bias of at most $2^{-r'}$.

By Lemma~\ref{lemma:min-extractor} (stated and proved below),
$M$ is an
$(\Omega(k) + \Omega(r),\Omega(\ell) + \Omega(r))$-$L_2$-extractor with error $2^{-\Omega(r)}$.

The corollary follows by Theorem~\ref{thm:TM-main1}.
\end{proof}

\section{Applications}

\subsection{Some Useful Lemmas}\label{sec:useful}
\subsubsection{Handling Negative Functions}
In the following lemma, we show that up to a small loss in parameters an $L_2$-extractor has similar guarantees for any function $f: \XX \to \Reals$ with bounded $\ell_2$-vs-$\ell_1$-norm regardless of whether or not $f$ is non-negative.

\begin{lemma}
Suppose that $M: \AA \times \XX \to \{-1,1\}$ is a $(k',\ell')$-$L_2$-extractor with error at most~$2^{-r}$.
Let $f: \XX \to \Reals$ be any function with $\frac{\norm{f}_2}{\norm{f}_1} \le 2^{\ell'-r}$.
Then, there are at most $2\cdot 2^{-k'}\cdot |A|$ rows  $a\in A$ with $\frac{|\inner{M_a, f}|}{\norm{f}_1} \ge 2\cdot 2^{-r}$.
\end{lemma}
\begin{proof}
	Let $f_{+}, f_{-}: \XX \to \Reals^{+}$ be the non-negative functions defined by
$$
f_{+}(x) =
\begin{cases}
 f(x),	& f(x)>0\\
 0,& \mbox{otherwise}
\end{cases}
\qquad\qquad
f_{-}(x) =
\begin{cases}
|f(x)|,	& f(x)<0\\
 0,& \mbox{otherwise}
\end{cases}
$$
for $x \in \XX$.
We have $f(x) =  f_{+}(x) - f_{-}(x)$ for all $x \in \XX$.
We split into two cases:
\begin{enumerate}
	\item If $\norm{f_{+}}_1 < 2^{-r} \cdot \norm{f}_1$, then
	$|\inner{M_a, f_{+}}| \le \norm{f_{+}}_1 < 2^{-r} \cdot \norm{f}_1$ for all $a\in \AA$.

\item If $\norm{f_{+}}_1 \ge 2^{-r} \cdot \norm{f}_1$,
then
$f_{+}$ is a non-negative function with
$$\frac{\norm{f_{+}}_2}{\norm{f_{+}}_1} \le \frac{\norm{f}_2}{\norm{f}_1 \cdot 2^{-r}} \le 2^{\ell'}\;.$$
Thus, we may use the assumption that $M$ is an $L_2$-extractor to deduce that there are at most $2^{-k'} \cdot |A|$ rows $a\in \AA$ with $|\inner{M_a, f_{+}}| \ge \norm{f_{+}}_1 \cdot 2^{-r}$.\end{enumerate}
In both cases, there are at most $2^{-k'} \cdot |A|$ rows $a\in \AA$ with $|\inner{M_a, f_{+}}| \ge \norm{f}_1 \cdot 2^{-r}$.
Similarly, there are at most $2^{-k'} \cdot |A|$ rows $a\in \AA$ with $|\inner{M_a, f_{-}}| \ge \norm{f}_1 \cdot 2^{-r}$.
Thus, for all but at most
$2\cdot 2^{-k'} \cdot |A|$ of the rows $a\in \AA$
we have
\[|\inner{M_a, f}| \le |\inner{M_a, f_{+}}| + |\inner{M_a, f_{-}}|   < 2\cdot \norm{f}_1 \cdot 2^{-r}\;.\qedhere
\]
\end{proof}

\subsubsection{Error vs. Min-Entropy}

\begin{lemma} \label{lemma:error-minent}
Let $M: \AA \times \XX \rightarrow \{-1,1\}$ be a matrix.
Let $k,\ell, r$ be such that any submatrix of $M$ of at least $2^{-k} \cdot |A|$ rows and at least
$2^{-\ell} \cdot |X|$ columns, has a bias of at most $2^{-r}$.

Then, there exist
$k' = k + \Omega(r)$,  $\; \ell' = \ell+ \Omega(r)$, and $ r' = \Omega(r)$,
such  that:
any submatrix of $M$ of at least $2^{-k'} \cdot |A|$ rows and at least
$2^{-\ell'} \cdot |X|$ columns, has a bias of at most $2^{-r'}$.
\end{lemma}

\begin{proof}
Assume without loss of generality that $k,\ell,r$ are larger than some sufficiently large absolute constant.

We will show that
there exists
$k' = k + \Omega(r)$,
such  that,
any submatrix of $M$ of at least $2^{-k'} \cdot |A|$ rows and at least
$2^{-\ell} \cdot |X|$ columns, has a bias of at most $2^{-\Omega(r)}$.
The proof of the lemma then follows by applying the same claim again on the transposed matrix.

Let $k' = k + \tfrac{r}{10}$. Assume for a contradiction that there exist $T \subseteq A$ of size at least $2^{-k'} \cdot |A|$ and $S \subseteq X$ of size at least
$2^{-\ell} \cdot |X|$, such that the bias of $T \times S$ is larger than, say,~$2^{-r/2}$.
By the assumption of the lemma, $|T| < 2^{-k} \cdot |A|$.

Let $T'$ be an arbitrary set of $2^{-k} \cdot |A|$ rows in $A \setminus T$.
By the assumption of the lemma, the bias of $T' \times S$ is at most $2^{-r}$.
Therefore, the bias of $(T' \cup T) \times S$ is at least
$$\tfrac{|T|}{|T' \cup T|} \cdot 2^{-r/2}  - \tfrac{|T'|}{|T' \cup T|} \cdot 2^{-r}
\geq  \tfrac{1}{2} \cdot  2^{-r/10} \cdot 2^{-r/2}  -  2^{-r} > 2^{-r}.$$

Thus, $(T' \cup T) \times S$ contradicts the assumption of the lemma.
\end{proof}
\subsubsection{$L_2$-Extractors and $L_\infty$-Extractors}

 We will show that $M$ being an $L_2$-Extractor is equivalent to $M$ being an $L_\infty$-Extractor (barring constants).

 \begin{lemma}\label{lem:L2toLinf}
 If a matrix $M:\AA\times \XX\rightarrow \{-1,1\}$ is a $(k,\ell)$-$L_2$-Extractor with error $2^{-r}$, then $M$ is also a $\left(k-\xi,2\ell \sim (\min\{r,\xi\}-1)\right)$-$L_\infty$-Extractor, $\forall 0<\xi<k$.
 \end{lemma}

 Taking $\xi=\frac{k}{2}$, we get that if $M$ is a $(k,\ell)$-$L_2$-Extractor with error $2^{-r}$, then $M$ is also a $\left(\Omega(k),\Omega(\ell) \sim (\Omega(\min\{r,k\}))\right)$-$L_\infty$-Extractor.

\begin{proof}
We pick a $\xi$ ($0<\xi<k$). To prove that $M$ is a $\left(k-\xi,2\ell \sim (\min\{r,\xi\}-1)\right)$-$L_\infty$-Extractor, it suffices to prove the statement of the $L_\infty$-Extractors for any two uniform distributions over subsets $A_1 \subseteq \AA$ and $X_1 \subseteq \XX$ of size at least $\frac{|\AA|}{2^{k-\xi}}$ and $\frac{|\XX|}{2^{2\ell}}$ respectively. This follows from the fact that any distribution with min-entropy at least $h$ can be written as a convex combination of uniform distributions on sets of size at least $2^h$~\cite{CG}.

For a distribution $p_x$, which is uniform over a subset $X_1 \subseteq \XX$ of size at least $\frac{|\XX|}{2^{2\ell}}$,
%
$$
\frac{\norm{p_x}_2}{\norm{p_x}_1}= \left(\frac{|\XX|}{|X_1|}\right)^{\frac{1}{2}}\le 2^{\ell}.
$$
Using the fact that $M$ is a $(k,\ell)$-$L_2$-Extractor with error $2^{-r}$, we know that there are at most $\frac{|\AA|}{2^k}$ rows $a$ with $|(M\cdot p_x)_a|\ge 2^{-r}$.
Using the fact that $p_a$ is a uniform distribution over a set $A_1$ of size at least $\frac{|\AA|}{2^{k-\xi}}$, we get
\begin{align*}
\left|\sum_{a' \in \AA}\sum_{x'\in \XX} {p_a(a') \cdot p_x(x') \cdot M(a',x')}\right|
&\le\frac{1}{|A_1|} \cdot \sum_{a'\in A_1} \left|(M\cdot p_x)_{a'}\right|\\
&\le\frac{1}{|A_1|} \cdot \left(\frac{|\AA|}{2^{k}} + |A_1|\cdot 2^{-r} \right) \le 2^{-\xi} + 2^{-r}
\end{align*}
This proves that $M$ is a $\left(k-\xi,2\ell \sim (\min\{r,\xi\}-1)\right)$-$L_\infty$-Extractor, $\forall 0<\xi<k$.
 \end{proof}

%

 \begin{lemma}\label{lemma:min-extractor}
 If a matrix $M:\AA\times \XX\rightarrow \{-1,1\}$ is a $\left(k,\ell \sim r\right)$-$L_\infty$-Extractor, then $M$ is also a $\left(k-1,\frac{\ell-\xi-1}{2}\right)$-$L_2$-Extractor with error $2^{-r} + 2^{-\xi + 1}$, $\forall 1\le\xi \le \ell-1$.
 \end{lemma}
 Taking $\xi=\frac{\ell}{2}$, we get that if $M$ is a $\left(k,\ell \sim r\right)$-$L_\infty$-Extractor, then $M$ is also a $(\Omega(k),\Omega(\ell))$-$L_2$-Extractor with error $2^{-\Omega(\min\{r,\ell\})}$.

 In this proof, we use the following notation. For two non-negative functions $P,Q:X~\rightarrow~\Reals$, we denote by $\mathrm{dist}(P,Q)$ the $\ell_1$-distance between the two functions, that is
\[\mathrm{dist}(P,Q)=\sum_{x\in X}|P(x)-Q(x)|\;.
\]
Note that $\mathrm{dist}(P,Q) = \norm{P-Q}_1 \cdot |X|$.

\begin{proof}
We want to prove that for any $1 \le \xi \le \ell-1$, and any non-negative function $f:\XX\rightarrow \Reals$ with $\frac{\norm{f}_2}{\norm{f}_1}\le 2^{\frac{\ell-\xi-1}{2}}$, there are at most $2\cdot 2^{-k}\cdot |\AA|$ rows $a\in \AA$ with $\frac{|\inner{M_a,f}|}{\norm{f}_1}\ge 2^{-r}+2^{-\xi+1}$.

\medskip
Let's assume that there exists a non-negative function $f:\XX\rightarrow \Reals$ for which the last statement is not true. Let $f_p$ be a probability distribution on $\XX$ defined by
$f_p(x)=\frac{f(x)}{\sum_{x}f(x)}=\frac{f(x)}{|\XX| \cdot \norm{f}_1}$.
Then,
$$
\norm{f_p}_2=\frac{\norm{f}_2}{|\XX|\cdot \norm{f}_1}\le \frac{2^{\frac{\ell-\xi-1}{2}}}{|\XX|}
$$
$$
\implies \left(\frac{\sum_xf_p(x)^2}{|\XX|}\right)^{\frac{1}{2}}\le \frac{2^{\frac{\ell-\xi-1}{2}}}{|\XX|}
$$
$$
\implies \sum_xf_p(x)^2 \le 2^{\ell-\xi-1-\log(|\XX|)}
$$
Thus, there is strictly less than $2^{-\xi}$ probability mass on elements $x$ with $f_p(x)> 2^{\ell-\log(|\XX|)-1}$. Let
$\bar{f_p}:\XX\rightarrow \Reals$
 be the trimmed function that takes values $f_p(x)$ at $x$ when $f_p(x)\le 2^{\ell-\log(|\XX|)-1}$ and 0 otherwise.
 We define a new probability distribution $p_x:\XX\rightarrow [0,1]$ as
 $$
 p_x(x')=\bar{f_p}(x')+\frac{1-\sum_{x'}\bar{f_p}(x')}{|\XX|}.
 $$
Informally, we are just redistributing the probability mass removed from $f_p$. It is easy to see that the new probability distribution $p_x$ has min-entropy at least $\log(|\XX|)-\ell$, 
 and
\begin{equation}
\label{eq:trimming}
 \mathrm{dist}(p_x,f_p)< 2^{-\xi+1}
 \end{equation}
as $\mathrm{dist}(p_x,f_p)\le \mathrm{dist}(p_x,\bar{f_p}) + \mathrm{dist}(\bar{f_p},f_p)< 2^{-\xi}+2^{-\xi}$.

Let $A_{\text{bad}}$ be the set of rows $a\in \AA$ with
$\frac{|\inner{M_a,f}|}{\norm{f}_1}=|(M\cdot f_p)_a|\ge 2^{-r}+2^{-\xi+1}$.
By our assumption, $|A_{\text{bad}}|\ge 2\cdot 2^{-k}|\AA|$. Let $A_1$ and $A_2$ be the set of rows $a$ with $(M\cdot f_p)_a\ge 2^{-r}+2^{-\xi+1}$ and $(M\cdot f_p)_a\le -(2^{-r}+2^{-\xi+1})$ respectively. As $A_{\text{bad}}=A_1\cup A_2$, w.l.o.g. $|A_1|\ge |A_{\text{bad}}|/2\ge 2^{-k}|\AA|$ (else we can work with $A_2$ and the rest of the argument follows similarly). Let $p_a$ be a uniform probability distribution over the set $A_1$. Clearly $p_a$ has min-entropy at least $\log(|\AA|)-k$.

As $(M\cdot f_p)_a\ge 2^{-r}+2^{-\xi+1}$ for the entire support of $p_a$, we get
\begin{equation}\label{eq:1}
	\left|\Ex_{a\in_R A_1}[ (M \cdot f_p)_a]\right| \ge 2^{-r}+2^{-\xi+1}.
\end{equation}
As the entries of $M$ have magnitude at most 1, 
we have
\begin{equation}\label{eq:2}
\left|\Ex_{a\in_R A_1}\left[ (M \cdot (p_x-f_p))_a\right]\right| \le
\Ex_{a\in_R A_1}\left[ \sum_{x'\in \XX}{|p_x(x')-f_p(x')|}\right] = \mathrm{dist}(p_x,f_p)\;.	
\end{equation}
Combining Equations~\eqref{eq:trimming}, \eqref{eq:1} and~\eqref{eq:2} together gives
$$
\left|\Ex_{a\in_R A_1}[ (M \cdot p_x)_a]\right| \;\ge\;
2^{-r}+2^{-\xi+1}-\mathrm{dist}(p_x,f_p) \;>\; 2^{-r}
$$
Thus, we have two distributions $p_a$ and $p_x$ with min-entropy at least $\log(|\AA|)-k$ and $\log(|\XX|)-\ell$ respectively  contradicting the fact that $M$ is a $\left(k,\ell\sim r\right)$-$L_\infty$-Extractor.
Hence no such $f$ exists and $M$ is a $(k-1,\frac{\ell-\xi-1}{2})$-$L_2$-Extractor with error $2^{-r}+2^{-\xi+1}$.
\end{proof}

\subsubsection{Transpose}
\begin{lemma}\label{lemma:transpose}
If a matrix $M:\AA\times \XX\rightarrow \{-1,1\}$ is a $(k,\ell)$-$L_2$-Extractor with error $2^{-r}$, then the transposed matrix $M^t$ is  an $(\Omega(\ell),\Omega(k))$-$L_2$-Extractor with error $2^{-\Omega(\min\{r,k\})}$.
\end{lemma}

\begin{proof}
	As $M$ is a $(k,\ell)$-$L_2$-Extractor with error $2^{-r}$, using Lemma \ref{lem:L2toLinf}, $M$ is also a $\left(\Omega(k),\Omega(\ell) \sim (\Omega(\min\{r,k\}))\right)$-$L_\infty$-Extractor.
The definition of $L_\infty$-Extractor is symmetric in its rows and columns and hence,
$M^t$ is also a $\left(\Omega(\ell),\Omega(k) \sim (\Omega(\min\{r,k\}))\right)$-$L_\infty$-Extractor.
Now, using Lemma \ref{lemma:min-extractor} on $M^t$, we get that
$M^t$ is also a $(\Omega(\ell),\Omega(k))$-$L_2$-Extractor with error $2^{-\Omega(\min\{r,k\})}$.
\end{proof}

\subsubsection{Lower Bounds for Almost Orthogonal Vectors}
In this section, we show that a matrix $M : \AA \times \XX \to \{-1,1\}$ whose rows are almost orthogonal is a good $L_2$-extractor.
A similar technique was used in many previous works (see for example~\cite{GS,CG,Alon95,Raz05}).
Motivated by the applications (e.g., learning sparse parities and learning from low-degree equations) in which some pairs of rows are not almost orthogonal, we relax this notion and only require that almost all pairs of rows are almost orthogonal.
We formalize this in the definition of $(\eps,\delta)$-almost orthogonal vectors.
\begin{definition}\label{def:eps-delta} {\bf $(\eps,\delta)$-almost orthogonal vectors:}
Vectors $v_1, \ldots, v_m \in \{-1,1\}^{X}$ are {\sf $(\eps, \delta)$-almost orthogonal} if for any $i \in [m]$ there are at most $\delta \cdot m$ indices $j\in [m]$ with $\left|\inner{v_i,v_j}\right| > \eps$.
\end{definition}
Definition~\ref{def:eps-delta} generalizes the definition of an $(\eps,\delta)$-biased set from~\cite{KRT}.
\begin{definition}\label{def:eps-delta-T}
{\bf $(\eps,\delta)$-biased set (\cite{KRT}):}
	A set $T \subseteq \{0,1\}^{n}$ is {\sf $(\eps, \delta)$-biased} if there are at most $\delta \cdot 2^n$ elements $a\in \{0,1\}^n$ with
	$\left|\Ex_{x\in_R T}[(-1)^{a \cdot x}]\right|> \eps$,
(where $a \cdot x $ denotes the inner product of $a$ and $x$, modulo 2).
\end{definition}
Definition~\ref{def:eps-delta-T} is a special case of Definition~\ref{def:eps-delta}, where the vectors corresponding to a set $T \subseteq \{0,1\}^n$ are defined as follows. With every  $a \in \{0,1\}^n$, we associate the vector $v_a$ of length $|T|$, whose $x$-th entry equals $(-1)^{a \cdot x}$ for any $x\in T$.
Indeed, $T$ is $(\eps,\delta)$-biased iff the vectors $\{v_a:a \in \{0,1\}^n\}$ are $(\eps,\delta)$-almost orthogonal.

\begin{lemma}[Generalized Johnson's Bound]\label{lemma:Johnson}
Let
$M \in \{-1,1\}^{A \times X}$ be a matrix.
Assume that $\{M_a\}_{a\in A}$ are $(\eps,\delta)$-almost orthogonal vectors.
Then, for any $\gamma > \sqrt{\eps}$ and any non-negative function $f:X \to \mathbb{R}^{+}$, we have at most $(\tfrac{\delta}{\gamma^2-\eps}) \cdot |A|$ rows $a \in A$ with \begin{equation*}
|\inner{M_a,f}|  \ge \gamma \cdot  \|f\|_2. \end{equation*}

In particular, fixing $\gamma = \sqrt{\eps + \delta^{1/2}}$, we have that $M$ is a $(k,\ell)$-$L_2$-extractor with error~$2^{-r}$,
for $k = \frac12 \log(1/\delta)$, and $\ell = r = \Omega\big(\min\{\log(1/\eps), \log(1/\delta)\}\big)$.
\end{lemma}
\begin{proof}
Fix $\gamma>\sqrt{\eps}$.
	Let $I_{+}$ (respectively, $I_{-}$) be the rows in $A$ with high correlation (respectively, anti-correlation) with $f$. More precisely:
\begin{align*}
I_{+} &:= \{i \in A: \;\; \inner{M_i,f} > \gamma \cdot \|f\|_2\}\;,\\
I_{-} &:= \{i \in A: \;\; -\inner{M_i,f} > \gamma \cdot \|f\|_2\}\;.	
\end{align*}
	Let $I = I_{+} \cup I_{-}$.
	Define $z = \sum_{i\in I_{+}}{M_i} - \sum_{i\in I_{-}}{M_i}$.
	We consider the inner product of $f$~and~$z$.
	We have
	\begin{align*}
	(|I| \cdot \gamma \cdot \|f\|_2)^2 < \inner{f,z}^2
	&= \Bigg(\Ex_{x \in_R X}\bigg[ f(x) \cdot \Big(\sum_{i\in I_{+}}{M_{i,x}}- \sum_{i\in I_{-}}{M_{i,x}}\Big)\bigg]\Bigg)^2\\
	&\le \Ex_{x\in_R X}\Big[ f(x)^2\Big]  \cdot \Ex_{x \in_R X}\Bigg[\Big(\sum_{i\in I_{+}} M_{i,x} -\sum_{i\in I_{-}} M_{i,x}\Big)^2\Bigg]\tag{Cauchy-Schwarz}\\
	&\le \|f\|_2^2 \cdot \sum_{i\in I} \sum_{i' \in I}{|\inner{M_{i},M_{i'}}|}.	
	\end{align*}
	For any fixed $i\in I$, we break the inner-sum $\sum_{i' \in I}{\left|\inner{M_{i},M_{i'}}\right|}$ according to whether or not $\left|\inner{M_i, M_{i'}}\right| > \eps $. By the assumption on $M$, there are at most $\delta\cdot |A|$ rows $i'$ for which the inner-product is larger than $\eps$. For these rows, the inner-product is at most $1$. Thus, we get
	\begin{align*}(|I| \cdot \gamma \cdot \|f\|_2)^2  &<\|f\|_2^2  \cdot \sum_{i \in I} \sum_{i' \in I}{|\inner{M_i,M_{i'}}|}\le \|f\|_2^2  \cdot |I|\cdot ( |A|\cdot \delta + \eps \cdot|I|).
	\end{align*}
That is,
\begin{align*}|I| \cdot \gamma^2  <  |A|\cdot \delta + \eps \cdot|I|.
	\end{align*}
	Rearranging gives
	\[|I| <  \left(\frac{\delta}{\gamma^2 - \eps}\right) \cdot |A|,\]
	which completes the first part of the proof.
		
	We turn to the in particular part.
Assume that $\frac{\norm{f}_2}{\norm{f}_1} \le 2^{\ell}$.
Thus, we proved that
there are at most  $\left(\frac{\delta}{\gamma^2 - \eps}\right) \cdot |A|$ rows  $a\in A$, such that,
\begin{equation*}
|\inner{M_a,f}|  \ge \gamma \cdot 2^{\ell} \cdot \|f\|_1.
\end{equation*}

Fixing $\gamma = \sqrt{\eps + \delta^{1/2}}$,
$k = \log(1/\delta^{1/2})$, and $\ell = r = \frac12 \log(1/\gamma)$,
we get that $M$ is a $(k,\ell)$-$L_2$-extractor with error $2^{-r}$
(Definition~\ref{definition:l2-extractor}).
Finally, note that $\ell = r = \Omega\big(\min\{\log(1/\delta), \log(1/\eps)\}\big)$, which completes the proof.
\end{proof}

\subsection{Learning  Sparse Parities}
\label{sec:sparse-parities}
As an application of Lemma~\ref{lemma:Johnson} and Theorem~\ref{thm:TM-main1}, we reprove the main result in \cite{KRT}.
\begin{lemma}\label{lemma:sparse}
Let $T\subseteq \{0,1\}^n$ be an $(\eps,\delta)$-biased set,
with $\eps \ge \delta$.
Define the matrix $M : \{0,1\}^n \times T \rightarrow \{-1, 1\}$
by $M(a,x) = (-1)^{a \cdot x}$.
Then, the learning task associated with~$M$ {\sf (``parity learning over $T$'')} requires either at least $\Omega(\log(1/\eps)\cdot \log(1/\delta))$ memory bits or at least $\poly(1/\eps)$ samples.
\end{lemma}
\begin{proof}
The rows $\{M_a\}_{a\in \{0,1\}^n}$ are $(\eps,\delta)$-almost orthogonal vectors.
Thus, by  Lemma~\ref{lemma:Johnson}, we get that $M$
is a $(k,\ell)$-$L_2$-extractor with
error~$2^{-r}$, for $k = \Omega(\log(1/\delta))$ and $r = \ell = \Omega(\log(1/\eps))$ (assuming $\eps \ge \delta$).
By  Theorem~\ref{thm:TM-main1}, we get the required memory-samples lower bound.
\end{proof}

\begin{lemma}[\cite{KRT}]\label{lemma:KRT}
There exists a (sufficiently small) constant $c>0$ such that the following holds.
Let $T_\ell = \{x \in \{0,1\}^n : \sum_{i}{x_i} = \ell\}$.
For any $\eps > (8\ell/n)^{\ell/2}$, $T_\ell$ is an $(\eps,\delta)$-biased set for $\delta = 2\cdot e^{-\eps^{2/\ell} \cdot n/8}$.
In particular, $T_\ell$ is an $(\eps,\delta)$-biased set for
\begin{enumerate}
	\item $\eps = 2^{-c \ell}$, $\delta = 2^{-c n}$, assuming $\ell \le cn$.
	\item $\eps = \ell^{-c \ell}$, $\delta = 2^{-c n/ \ell^{0.01}}$, assuming $\ell \le n^{0.9}$.
\end{enumerate}
\end{lemma}

Let $c>0$ be the constant mentioned in Lemma~\ref{lemma:KRT}. The following lemma complements Lemma~\ref{lemma:KRT} to the range of parameters $cn \le \ell \le n/2$. It shows that $T_{\ell}$ is $(2^{-\Omega(n)}, 2^{-\Omega(n)})$-biased in this case. The proof is a simple application of Parseval's identity (see \cite{KRT}).
\begin{lemma}[\protect{\cite[Lemma~4.1]{KRT}}]
\label{lemma:KRT2}
Let $T \subseteq \{0,1\}^n$ be any set. Then, $T$ is an $(\eps,\delta)$-biased set for $\delta = \frac{1}{|T|\cdot \eps^2}$. In particular, $T$ is $(|T|^{-1/3}, |T|^{-1/3})$-biased.
\end{lemma}

We get the following as an immediate corollary.
\begin{corollary}
	Let $T_\ell = \{x \in \{0,1\}^n : \sum_{i}{x_i} = \ell\}$.
\begin{enumerate}
	\item Assuming $\ell \le  n/2$, parity learning over $T_\ell$ requires either at least $\Omega(n \cdot \ell)$ memory bits or at least $2^{\Omega(\ell)}$ samples.
	\item Assuming $\ell \le n^{0.9}$,  parity learning over $T_\ell$ requires either at least $\Omega(n \cdot \ell^{0.99})$ memory bits or at least $\ell^{\Omega(\ell)}$ samples.
\end{enumerate}
\end{corollary}

\subsection{Learning from Sparse Linear Equations}
\label{sec:sparse-equations}
Lemma~\ref{lemma:transpose} and the proof of
Lemma~\ref{lemma:sparse} gives the following immediate corollary.


\begin{lemma}\label{lemma:sparse-transpose}
Let $T\subseteq \{0,1\}^n$ be an $(\eps,\delta)$-biased set, with $\eps \ge \delta$.
Then, the matrix $M : T \times \{0,1\}^n \to \{-1, 1\}$,
defined by $M(a,x) = (-1)^{a \cdot x}$ is a $(k,\ell)$-$L_2$-extractor with error $2^{-r}$, for $\ell = \Omega(\log(1/\delta))$ and $k = r = \Omega(\log(1/\eps))$.

Thus, the learning task associated with $M$ {\sf(``learning from equations in $T$'')} requires either at least
$\Omega(\log(1/\eps)\cdot \log(1/\delta))$ memory bits or
at least $\poly(1/\eps)$ samples.
\end{lemma}
We get the following as an immediate corollary of Lemmas~\ref{lemma:KRT}, \ref{lemma:KRT2} and~\ref{lemma:sparse-transpose}.
\begin{corollary}
	Let $T_\ell = \{x \in \{0,1\}^n : \sum_{i}{x_i} = \ell\}$.

\begin{enumerate}
\item	Assuming $\ell \le  n/2$, learning from equations in $T_\ell$ requires either at least $\Omega(n \cdot \ell)$ memory bits or at least $2^{\Omega(\ell)}$  samples.
\item Assuming $\ell \le n^{0.9}$, learning from equations in $T_\ell$ requires either at least $\Omega(n \cdot \ell^{0.99})$ memory bits or at least $\ell^{\Omega(\ell)}$  samples.
\end{enumerate}
\end{corollary}

\subsection{Learning from Low Degree Equations}
\label{sec:low-deg-equations}
In the following, we consider multilinear polynomials in $\F_2[x_1, \ldots, x_n]$ of degree at most $d$. We denote by $P_d$ the linear space of all such polynomials.
We denote the bias of  a polynomial $p \in \F_2[x_1, \ldots, x_n]$  by $$\bias(p) := \Ex_{x\in \F_2^n}[(-1)^{p(x)}].$$

We rely on the following result of Ben-Eliezer, Hod and Lovett \cite{BEHL}, showing that random low-degree polynomials have very small bias with very high probability.
\begin{lemma}[\protect{\cite[Lemma~2]{BEHL}}]
Let $d \le 0.99 \cdot n$. Then,
\[
\Pr_{p\in_R P_d}[|\bias(p)| > 2^{-c_1 \cdot n/d}] \le 2^{-c_2 \cdot \binom{n}{\le d}}
\]
where $0<c_1,c_2<1$ are absolute constants.
\end{lemma}

\begin{corollary}
\label{cor:low-deg-equations}
Let $d, n \in \N$, with $d\le 0.99 \cdot n$.
Let $M: P_d \times \F_2^n \to \{-1,1\}$ be the matrix defined by $M(p,x) = (-1)^{p(x)}$ for any $p\in P_d$ and $x \in \F_2^n$.
Then, the vectors $\{M_p: p \in P_d\}$ are
$(\eps,\delta)$-almost orthogonal, for
$\eps = 2^{-c_1 n/d}$ and
$\delta = 2^{-c_2\binom{n}{\le d}}$,
(where $0<c_1,c_2<1$ are absolute constants).
In particular, $M$ is a $(k,\ell)$-$L_2$-extractor with error $2^{-r}$, for
$k = \Omega\big(\binom{n}{\le d}\big)$ and
$r = \ell = \Omega(n/d)$.

Thus,
the learning task associated with $M$
{\sf (``learning from degree-$d$ equations'')}
requires either at least
$\Omega\left( \binom{n}{\le d} \cdot n/d \right) \ge \Omega((n/d)^{d+1})$
memory bits or at least
$2^{\Omega(n/d)}$
samples.
\end{corollary}

\begin{proof}
	We reinterpret \cite[Lemma~2]{BEHL}.
	Since $P_d$ is a linear subspace, for  any fixed
	$p \in P_d$ and a uniformly random $q \in_R P_d$, we have that $p+q$ is  a uniformly random polynomial in~$P_d$.
	Thus, for any fixed $p\in P_d$,
	at most
	$2^{-c_2 \cdot \binom{n}{\le d}}$ fraction of the
	polynomials $q \in P_d$ have
	$$|\bias(p+q)| \ge 2^{-c_1 \cdot n/d}.$$
	In other words, we get that $\{M_p : p \in P_d\}$ are $(\eps, \delta)$-almost orthogonal vectors for $\eps = 2^{-c_1 \cdot n/d}$ and $\delta = 2^{-c_2\cdot \binom{n}{\le d}}$.
	We apply Lemma~\ref{lemma:Johnson} to get the ``in particular'' part, noting that in our case
	$\Omega\big(\min\{\log(1/\eps), \log(1/\delta)\}\big) = \Omega(n/d)$.
	We apply Theorem~\ref{thm:TM-main1} to get the ``thus'' part.
\end{proof}

\subsection{Learning Low Degree Polynomials} \label{sec:low-degree-polynomials}
Lemma~\ref{lemma:transpose} and Corollary~\ref{cor:low-deg-equations} gives the following immediate corollary.

\begin{corollary}
\label{cor:low-deg}
Let $d, n \in \N$, with $d\le 0.99 \cdot n$.
Let $M: \F_2^{n} \times P_d \to \{-1,1\}$ be the matrix defined by $M(a,p) = (-1)^{p(a)}$ for any $p\in P_d$ and $a \in \F_2^{n}$.
Then, $M$ is a $(k,\ell)$-$L_2$-extractor with error $2^{-r}$, for
$\ell = \Omega\big(\binom{n}{\le d}\big)$ and $k = r = \Omega(n/d)$.

Thus,
the learning task associated with $M$ {\sf (``learning degree-$d$ polynomials'')} requires either at least
$\Omega\left( \binom{n}{\le d} \cdot n/d \right) \ge \Omega((n/d)^{d+1})$
memory bits or at least
$2^{\Omega(n/d)}$
samples.
\end{corollary}

\subsection{Relation to Statistical-Query-Dimension} \label{sec:sq}
Let $\mathcal C$ be a class of functions mapping $\AA$ to $\{-1,1\}$.
The {\sf Statistical-Query-Dimension} of~$\mathcal C$, denoted $\mathrm{SQdim}(\mathcal C)$, is defined to be the maximal $m$ such that there exist functions $f_1, \ldots, f_m \in \mathcal{C}$ with $|\inner{f_i,f_j}| \le 1/m$ for all $i \neq j$~\cite{K98,BFJKMR}.
As a corollary of Lemma~\ref{lemma:transpose} and Lemma~\ref{lemma:Johnson}, we get the following.

\begin{corollary}
Let $\mathcal C$ be a class of functions mapping $\AA$ to $\{-1,1\}$.
Let $\mathrm{SQdim}(\mathcal C) = m$.
Let $f_1, \ldots, f_m \in \mathcal {C}$ with
$|\inner{f_i,f_j}| \le 1/m$ for any $i \neq j$.
Define the matrix $M: \AA \times [m] \to \{-1,1\}$ whose columns are the vectors $f_1, \ldots, f_m$.
Then, $M$ is a $(k,\ell)$-$L_2$-extractor with error $2^{-r}$ for $k = \ell = r = \Omega(\log m)$.

Thus,
the learning task associated with $M$  requires either at least
$\Omega( \log^2 m)$
memory bits or at least
$m^{\Omega(1)}$
samples.\end{corollary}
\begin{proof}
Consider the rows of the matrix $M^{t}$.
By our assumption, the rows of $M^t$ are $(1/m,1/m)$-almost orthogonal.
Thus, by Lemma~\ref{lemma:Johnson}, $M^t$ is a $(k,\ell)$-$L_2$-extractor with error~$2^{-r}$, for $k = \ell = r = \Omega(\log m)$.
By Lemma~\ref{lemma:transpose}, $M$ is a $(k,\ell)$-$L_2$-extractor with error $2^{-r}$ for $k = \ell = r = \Omega(\log m)$.
	We apply Theorem~\ref{thm:TM-main1} to get the ``thus'' part.
\end{proof}

In fact, we get the following (slight) generalization.
Suppose that there are $m'\ge m$ functions $f_1, \ldots, f_{m'}$ mapping $\AA$ to $\{-1,1\}$ with $|\inner{f_i, f_j}|\le 1/m$ for all $i\neq j$.
Then, the learning task associated with the matrix whose columns are $f_1, \ldots, f_{m'}$  requires either at least $\Omega(\log(m) \cdot  \log (m'))$ memory bits or at least $m^{\Omega(1)}$ samples.

\subsection{Comparison with~\cite{Raz17}}

\paragraph{Small Matrix Norm implies $L_2$-Extractor.}
This paper generalizes the result of~\cite{Raz17}
that if a matrix $M:\AA\times \XX\rightarrow \{-1,1\}$ is such that the largest singular value of $M$, $\sigma_{\max}(M)$, is at most $|\AA|^{\frac{1}{2}}|\XX|^{\frac{1}{2}-\varepsilon}$, then the learning problem represented by $M$
requires either a memory of size at least
$\Omega \left( (\varepsilon n)^2  \right) $
or at least
$2^{\Omega(\varepsilon n)}$ samples, where
$n = \log_2|\XX|$. We use the following lemma:
\begin{lemma}\label{lem:smallnorm}
If a matrix $M:\AA\times \XX\rightarrow \{-1,1\}$ satisfies $\sigma_{\max}(M) \le |\AA|^{\frac{1}{2}} \cdot |\XX|^{\frac{1}{2}-\varepsilon}$, then $M$ is a $(k,\ell)$-$L_2$-Extractor with error $2^{-r}$ for every $k,\ell,r>0$ such that $k+2\ell+2r \le 2\varepsilon n$ ($n=\log_2(|\XX|)$).
\end{lemma}
Theorem \ref{thm:TM-main1} and Lemma \ref{lem:smallnorm} with $k=\varepsilon n, \ell=r=\frac{\varepsilon n}{4}$, imply
the main result of~\cite{Raz17}.
\begin{proof}
 As  $\sigma_{\max}(M) \le |\AA|^{\frac{1}{2}}|\XX|^{\frac{1}{2}-\varepsilon}$, for a non-negative function $f:\XX\rightarrow \Reals$,
$\norm{M\cdot f}_2\le |\XX|^{1-\varepsilon}\cdot \norm{f}_2$. In other words,
$$
\left( \Ex_{a \in_R \AA} \left[ |(M\cdot f)_a|^{2} \right] \right)^{1/2}\le |\XX|^{1-\varepsilon} \cdot \norm{f}_2
$$
$$
\implies \left( \Ex_{a \in_R \AA} \left[ |\inner{M_a,f}|^{2} \right] \right)^{1/2}\le |\XX|^{-\varepsilon} \cdot \norm{f}_2
$$
$$
\implies \left( \Ex_{a \in_R \AA} \left[ \left(\frac{|\inner{M_a,f}|}{\norm{f}_1}\right)^{2} \right] \right)^{1/2}\le 2^{-\varepsilon n}\cdot \frac{\norm{f}_2}{\norm{f}_1}
$$
Now if $\frac{\norm{f}_2}{\norm{f}_1}\le 2^{\ell}$ for some $\ell>0$, then
$$
  \Ex_{a \in_R \AA} \left[ \left(\frac{|\inner{M_a,f}|}{\norm{f}_1}\right)^{2} \right]\le 2^{-2\varepsilon n+2\ell}\;.
$$
Applying Markov's inequality, we get that there are at most $2^{-2\varepsilon n + 2\ell +2r}\cdot |\AA|$ rows $a\in \AA$ with
$\frac{|\inner{M_a,f}|}{\norm{f}_1} \ge 2^{-r}$.
\end{proof}

\subsection{Comparison with~\cite{MM2}} \label{sec:MM}

We will now show that our result subsumes the one of~\cite{MM2}.
Moshkovitz and Moshkovitz~\cite{MM2} consider matrices
$M: \AA \times \XX \rightarrow \{-1,1\}$, and a parameter $d$, with the property that
for any  $A' \subseteq A$ and $X' \subseteq X$ the bias of the submatrix
$M_{A' \times X'}$ is at most $\tfrac{d}{\sqrt{|A'| \cdot |X'|}}$.
They define $m = \tfrac{|A| \cdot |X|}{d^2}$ and prove that any learning
algorithm for the corresponding learning problem  requires either a memory  of size $\Omega((\log m)^2)$ or $m^{\Omega(1)}$ samples.
We note that this is essentially the same result as the one proved in~\cite{Raz17}, and since it is always true that
$d^2 \geq \max\left\{|X|,|A|\right\}$, the bound obtained on the memory is at most
$\Omega\left(\min\left\{(\log |X|)^2,(\log |A|)^2\right\}\right)$.

Note that if $M$ satisfies that property (required by~\cite{MM2}), then,
in particular,
any submatrix $A' \times X'$ of $M$ of at least $m^{-1/4} \cdot |A|$ rows and at least
$m^{-1/4} \cdot |X|$ columns, has a bias of at most
$$
\tfrac{d}{\sqrt{|A'| \cdot |X'|}} =
\tfrac{d} {\sqrt{|A| \cdot |X|}} \cdot
\tfrac{\sqrt{|A| \cdot |X|}}{\sqrt{|A'| \cdot |X'|}}
\leq
m^{-1/2} \cdot m^{1/4} = m^{-1/4}.
$$

Thus, we can apply Corollary~\ref{cor:main1}, with
$k,\ell, r = \tfrac{1}{4} \log(m)$ to obtain the same result.

\subsection*{Acknowledgement}
We would like to thank Pooya Hatami and Avi Wigderson for very helpful conversations.

\end{document}